\documentclass{amsart} 

\usepackage{latexsym}
\usepackage{amssymb}
\usepackage{amsmath}
\usepackage{amsthm}
\usepackage{booktabs}
\usepackage{enumitem}
\usepackage{graphicx}
\usepackage{color}

\usepackage{amsmath,amssymb,amsfonts}

\usepackage{algorithmic}
\usepackage{graphicx}
\usepackage{textcomp}
\usepackage[table]{xcolor}

\definecolor{myColor}{RGB}{66,245,236}

\usepackage{thmtools}
\usepackage{thm-restate}

\usepackage{tikz}
\usetikzlibrary{arrows.meta, positioning}

\usepackage{amsthm}

\usepackage{hyperref}
\hypersetup{
    colorlinks,
    linkcolor={red!50!black},
    citecolor={blue!50!black},
    urlcolor={blue!80!black}
}

\usepackage{cleveref}

\newtheorem{theorem}{Theorem}

\newtheorem{definition}[theorem]{Definition}
\newtheorem{lemma}[theorem]{Lemma}
\newtheorem{proposition}[theorem]{Proposition}

\newtheorem{remark}[theorem]{Remark}

\providecommand{\F}{\mathcal{F}}
\providecommand{\G}{\mathcal{G}}
\renewcommand{\H}{\mathcal{H}}
\providecommand{\I}{\mathcal{I}}

\providecommand{\C}{\mathcal{C}}
\providecommand{\T}{\mathcal{T}}

\renewcommand{\hat}{\widehat}
\renewcommand{\phi}{\varphi}

\providecommand{\att}{\rightarrowtail}
\providecommand{\natt}{\not\rightarrowtail}

\providecommand{\nat}{\mathbb{N}}
\renewcommand{\omega}{\nat}

\providecommand{\Lo}{\textbf{L}}

\providecommand{\Po}{\textbf{P}}
\providecommand{\NP}{\textbf{NP}}
\providecommand{\coNP}{\textbf{coNP}}
\providecommand{\DP}{\textbf{DP}}
\providecommand{\IH}{\textbf{IH}}

\providecommand{\Inf}{\textbf{Inf}}
\providecommand{\Fin}{\textbf{Fin}}
\providecommand{\tot}{\textbf{Tot}}

\providecommand{\ad}{\textit{ad}}
\providecommand{\stb}{\textit{stb}}
\providecommand{\co}{\textit{co}}
\providecommand{\infad}{\textit{infad}}
\providecommand{\infstb}{\textit{infstb}}
\providecommand{\infco}{\textit{infco}}
\providecommand{\cf}{\textit{cf}}
\providecommand{\na}{\textit{na}}
\providecommand{\infcf}{\textit{infcf}}
\providecommand{\infna}{\textit{infna}}

\providecommand{\ins}{\textit{In}}
\providecommand{\out}{\textit{Out}}

\providecommand{\range}{\text{range}}

\providecommand{\cred}{\text{Cred}}
\providecommand{\skep}{\text{Skept}}
\providecommand{\ex}{\text{Exist}}
\providecommand{\nemp}{\text{NE}}
\providecommand{\uni}{\text{Uni}}

\providecommand{\icf}{{\text{cfin}}}

\title[Complexity in Finitary Argumentation (extended version)]{Complexity in Finitary Argumentation\\ (extended version)}

\author[Uri Andrews]{Uri Andrews}
\address{Department of Mathematics, University of Wisconsin-Madison}
\email{\href{mailto:andrews@math.wisc.edu}{andrews@math.wisc.edu}}

\author[Luca San Mauro]{Luca San Mauro}
\address{Department of Humanistic Research and Innovation, University of Bari}
\email{\href{mailto:luca.sanmauro@gmail.com}{luca.sanmauro@gmail.com}}

\begin{document}

\begin{abstract}
Abstract argumentation frameworks (AFs) provide a formal setting to analyze many forms of reasoning with conflicting information. 
While the expressiveness of general infinite AFs make them a tempting tool for modeling many kinds of reasoning scenarios, the computational intractability of solving infinite AFs limit their use, even in many theoretical applications.

We investigate the complexity of computational problems related to infinite but finitary argumentations frameworks, that is, infinite AFs where each argument is attacked by only finitely many others. 
Our results reveal a surprising scenario. 
On one hand, we see that the assumption of being finitary does not automatically guarantee a drop in complexity. However, for the admissibility-based semantics, we find a remarkable combinatorial constraint which entails a dramatic decrease in complexity. 

We conclude that for many forms of reasoning, the finitary infinite AFs provide a natural setting for reasoning which balances well the competing goals of being expressive enough to be applied to many reasoning settings while being computationally tractable enough for the analysis within the framework to be useful.
\end{abstract}


\thanks{This work was supported by the the National Science Foundation under Grant DMS-2348792. San Mauro is a member of INDAM-GNSAGA}

\maketitle

\section{Introduction}

The study of abstract arguments and their interactions originated in Dung's seminal work \cite{dung1995acceptability} and has since become a central area of research within Artificial Intelligence (for a rich survey of this area and its ramifications, see the handbook \cite{Baroni2018-BARHOF}). In particular, abstract argumentation frameworks (AFs) allow us to model and process a wide range of reasoning problems. The core idea is that arguments, whether arising from dialogues between multiple agents or from information available to a single agent, can conflict through attacks, and abstract argumentation offers a structured way to manage these conflicts.

A key challenge for any AF is determining which sets of arguments, known as \emph{extensions}, are acceptable. Different reasoning contexts necessitate different acceptability criteria, which has resulted in the development of a wide range of argumentation \emph{semantics} that prescribe which extensions of a given AF are deemed acceptable. 
For each semantics, there are natural associated computational problems, such as determining whether an argument belongs to some accepted extension (credulous acceptance) or to all accepted extensions (skeptical acceptance). Determining the computational complexity of these problems has been a major research direction in formal argumentation \cite{dimopoulos1996graph,dunne2002coherence,dunne2009complexity} (see Table \ref{finiteTable} for collected results).

Infinite AFs played a prominent role in Dung's seminal work and arise in multiple practical domains, but they have received far less attention than their finite counterparts. In application-focused research, where direct solvability is often crucial, this focus on the finite is understandable. Yet restricting to finite frameworks sacrifices generality and overlooks settings where infinite structures are unavoidable. Dung's own encoding of logic programming already produces infinite, and sometimes non-finitary, AFs \cite{dung1995acceptability}, and later work has identified many further examples \cite{baumann2017study,baroni2013automata,flouris2019comprehensive}, including continuum-sized frameworks from cooperative game theory \cite{young2020continuum}. Infinite AFs also arise naturally in structured argumentation systems such as ASPIC+ \cite{modgil2014aspic+} and assumption-based argumentation \cite{toni2014tutorial}, where expressive languages can generate infinitely many arguments even from finite knowledge bases. Finally, domains such as science, law, and AI-driven dialogue involve evolving, unbounded reasoning processes; here, infinite AFs provide a natural idealization, supporting analyses of reasoning “in the limit” \cite{belardinelli2015formal} and offering insights into scalability, convergence, and complexity. 



There has been an increase in interest in infinite AFs in recent years, with particular attention paid to how the existence and interaction of different semantics are influenced in the infinite domain \cite{baroni2013automata,baumann2015infinite,caminada2014grounded-inf,baumann2017study}. Recently, Andrews and San Mauro \cite{andrews2024NMR, andrews2024FCR} developed a paradigm, rooted in computability theory, which is suitable to analyze the computational problems of credulous and skeptical acceptance of arguments in infinite AFs. They defined an infinite AF as computable if its attack relation is computable, meaning there exists a Turing machine that can decide, for any pair of arguments, whether one attacks the other. By studying the complexity of the computational problems of credulous and skeptical acceptance of arguments on the computable AFs, we can isolate the complexity of the reasoning inside the AF from the complexity of the AF itself.
As for the case of many of the computational problems in the finite setting, their infinite analogues are, if non-trivial, maximally hard (see Table \ref{infinitaryTable} for some collected results).

With these complexity results in mind, we are faced with a familiar dichotomy: We would like to work with infinite AFs due to their vast expressiveness and usefulness for a myriad of reasoning tasks, but their computational problems are intractably hard. In particular, there is no computation which, at any finite time, yields any useful evidence towards accepting or not accepting an argument. 

In this paper, we offer an approach to manage this situation by considering
a fundamental class of AFs which dates back to the work of Dung \cite[Definition 27]{dung1995acceptability}: namely, infinite but \emph{finitary} AFs, where each argument is attacked by at most finitely many others. This class is still quite expressive with application. For example, \cite[Section 7.3,7.4]{baroni2013automata} show applications of finitary AFs in the contexts of multi-agent negotiation and in ambient intelligence.
Yet, we show that the computational problems associated to many familiar semantics are surprisingly simple. In particular, though these problems are still incomputable, we can perform calculations which, at finite time, provide evidence for accepting or rejecting an argument. For example, we show that some problems are limits of computable functions and others are $\limsup$s of computable functions.

This approach to handling computational intractability is familiar also in the finite setting. In the finite setting, researchers have identified specific classes of directed graphs, such as symmetric, acyclic, or bipartite directed graphs \cite{coste2005symmetric,dunne2007computational,dvovrak2012augmenting}, that allow more efficient evaluation of AFs. 


The findings presented in this paper, collected in Table \ref{finitaryTable}, reveal a surprising scenario. On one hand, we see that the assumption of being finitary does not yield a drop in complexity for all semantics. However, for the admissibility-based semantics, we find a remarkable combinatorial constraint which entail a dramatic decrease in complexity. This is most surprising when we consider the infinite semantics
(see Definition \ref{def: infinite semantics}), which were introduced in \cite{andrews2024FCR, andrews2024NMR} motivated by the idea that infinite AFs should yield an infinite amount of information. 

In Section \ref{sec:background}, we review the background, both from argumentation theory and computability theory needed in this paper. In Section \ref{sec:comply finitary AFs}, we formally define the complexity problems which we solve in this paper. Specifically, we introduce the computably finitary AFs. In Section \ref{sec:results}, we give a table describing all the results proved in this paper, which are then detailed in the remainder. The omitted proofs are provided in the supplementary material (Section \ref{supplement}). 


\section{Background}\label{sec:background}
\subsection{Argumentation theoretic background}\label{sec:AF Background}

We briefly review some key concepts of Dung-style   argumentation
theory, focusing on the  semantics notions considered in this paper and the fundamental computational
problems associated with them (the surveys \cite{baroni2009semantics,dunne2009complexity} offer an overview of these topics). 

An \emph{argumentation framework} (AF) $\F$  is a pair $(A_\F,R_\F)$ consisting of  a set $A_\F$ of arguments and an attack relation $R_\F\subseteq A_\F\times A_\F$. If some argument $a$ attacks some argument $b$, 
 we often write $a\att b$ instead of $(a, b)\in R_\F$.   Collections of arguments $S\subseteq A_\F$ are called \emph{extensions}. For an extension $S$, the symbols $S^+$ and $S^-$ denote, respectively, the arguments that $S$ attacks and the arguments that attack $S$: 
\[
S^+=\{x : (\exists y \in  S)(y \att x)\};
S^-=\{x : (\exists y \in  S)(x \att y) \}.
\]
$S$ \emph{defends} an argument $a$, if  any argument that attacks $a$ is attacked by some argument in $S$ (i.e., $\{a\}^-\subseteq S^+$). The \emph{characteristic function} of $\F$ is the mapping $f_\F$ which sends subsets of $A_\F$ to subsets of $A_\F$ via 
$f_\F(S) := \{x :  x \text{ is defended by $S$}\}$. 
An AF
$\F$ is \emph{finitary} if  $\{x\}^-$ is finite for all $x\in A_\F$.

A \emph{semantics} $\sigma$  assigns to every AF $\F$ a set of extensions $\sigma(\F)$ which are  deemed as acceptable.  Several semantics, fueled by different motivations, have been proposed and analyzed.  Here, we focus on five prominent choices, whose computational aspects are well-understood in the finite setting: conflict-free, naive, admissible, complete, and stable semantics (abbreviated by $\cf, \na, \ad, \co, \stb$, respectively). Let $\F=(A_\F,R_\F)$ be an AF. Denote by $\cf(\F)$ the collection of extensions of $\F$ which are \emph{conflict-free} (i.e., $S\in \cf(\F)$ iff  $a \natt b$, for all $a,b\in S$). Then, for $S\in \cf(\F)$,
\begin{itemize}
\item $S\in \na(\F)$ iff there is no $S'\supsetneq S$ which is conflict-free.
\item $S\in \ad(\F)$ iff $S$ is self-defending (i.e., $S\subseteq f_\F(S)$);
\item $S\in \stb(\F)$ iff $S$ attacks all arguments outside of itself (i.e., $S^+= A_\F\smallsetminus S$);
\item $S\in \co(\F)$ iff $S$ is a fixed point of $f_\F$ (i.e., $S= f_\F(S)$).
\end{itemize}
For each of these semantics it is natural to ask for infinite extensions within this semantic. 
\begin{definition}\label{def: infinite semantics}
For each $\sigma\in \{\cf,\na,\ad,\co,\stb\}$, we also consider $\textit{inf}\sigma$, where $S\in \textit{inf}\sigma(\F)$ iff $S\in \sigma(\F)$ and $S$ is infinite.
\end{definition}

As an illustration of why we might want to accept only infinite extensions, we consider that a given infinite AF may contain a single argument $b$ so that $b$ attacks every other argument, and every other argument attacks $b$. We imagine that $b$ is a statement of extreme solipsism denying the truth of any other statement. While $\{b\}$ is a stable extension, it represents a negligible fraction of arguments, and we may prefer not to accept it. In an infinite AF, any finite set is as negligible as $\{b\}$, so we may prefer to accept only infinite extensions.

For a given semantics $\sigma$, the following are some well-known computational problems related to $\sigma$:
\begin{itemize}
    \item \textbf{Credulous acceptance} ($\cred_\sigma$): Given a pair $(\F, a)$, decide whether $a$ belongs to \emph{some} extension $S \in \sigma(\F)$.
    
    \item \textbf{Skeptical acceptance} ($\skep_\sigma$): Given a pair $(\F, a)$, decide whether $a$ belongs to \emph{every} extension $S \in \sigma(\F)$.
    
    \item \textbf{Extension existence} ($\ex_\sigma$): Given an argumentation framework $\F$, determine whether $\sigma(\F)$ is non-empty.
    
    \item \textbf{Non-empty extension existence} ($\nemp_\sigma$): Given $\F$, decide whether $\sigma(\F)$ contains at least one \emph{non-empty} extension.
    
    \item \textbf{Uniqueness of extension} ($\uni_\sigma$): Given $\F$, determine whether there exists exactly one extension under $\sigma$, i.e., $|\sigma(\F)| = 1$.
\end{itemize}

In formal argumentation theory, evaluating the computational complexity of the aforementioned problems for various semantics has been a noteworthy research thread for more than 20 years \cite{dunne2009complexity}. Table \ref{finiteTable}, which appears in \cite{dvorak2018computational}, collects known complexity results for finite AFs and these semantics. In the next section, we introduce the computability theoretic machinery that enable tackling complexity issues concerning infinite AFs.

\begin{table}[ht]
\caption{Complexity of computational problems for finite AFs. $\mathcal{C}$-c denotes completeness for $\mathcal{C}$.}
	\centering
	\begin{tabular}{|l|l|l|l|l|l|}
	
		\hline
		$\sigma$ & $\cred_\sigma$ & $\skep_\sigma$ & $\ex_\sigma$ & $\nemp_\sigma$ & $\uni_\sigma$ \\

            \hline
		$\cf$ & \Lo & trivial  & trivial &  \Lo  & \Lo  \\
		\hline

            \hline
		$\na$ & \Lo & \Lo  & trivial & \Lo  & \Lo  \\
		\hline
  
		\hline
		$\ad$ & \NP-c & trivial  & trivial & \NP-c  & \coNP-c  \\
		\hline

		$\stb$ & \NP-c &  \coNP-c & \NP-c & \NP-c & \DP-c  \\
		\hline

		$\co$ & \NP-c  & \Po-c & trivial  & \NP-c  & \coNP-c  \\
		\hline
	\end{tabular}
	\label{finiteTable}
\end{table}

\subsection{Computability theoretic background}
In this section, we offer a succinct summary of the computability theoretic notions needed to assess the complexity of computational problems for infinite AFs. 
A more formal and comprehensive
exposition of the fundamentals of computability theory
can be found, e.g., in the textbook \cite{rogers1987theory}.

\subsubsection{Numbers, strings, and trees}
We denote the set of natural numbers by $\omega$. 
In order to formulate our problems as subsets of $\omega$, it will be convenient to encode pairs of numbers into single numbers. The pairing function does this. Fix $p: \omega\times \omega\to \omega$ 
 to be a computable bijection. We adopt the common practice of denoting $p(x,y)$ by $\langle x,y\rangle$. 
 Given a finite set $A = \{x_1, x_2, \ldots , x_n\}\subset \nat$, we define $y = 2^{x_1} + 2^{x_2} + \ldots + 2^{x_n}$ to be the \emph{canonical index} of $A$; let $D_y$ denote
the finite set with canonical index $y$. Using canonical indices, we can quantify over all finite subsets of $A_\F$. 


The set of all finite strings of natural numbers is denoted by $\omega^{<\omega}$, and the set of infinite strings of natural numbers is denoted by $\nat^\nat$.  
The concatenation of strings $\sigma,\tau$ is denoted by $\sigma^\smallfrown \tau$. The length of a string $\sigma$ is denoted by $|\sigma|$. If there is $\rho$ so that $\sigma^\smallfrown \rho=\tau$, we say that $\sigma$ is a \emph{prefix} of $\tau$ and  we write $\sigma \preceq\tau$. If $\sigma$ is not a prefix of $\tau$ and $\tau$ is not a prefix of $\sigma$, then we say $\sigma$ and $\tau$ are incomparable.   
Similarly, if $\pi\in \omega^{\omega}$ and $\sigma$ is a prefix of $\pi $, we write $\sigma\prec \pi$.

A \emph{tree} is a set $\T\subseteq \omega^{<\omega}$ closed under prefixes. We picture trees growing upwards, with $\sigma^\smallfrown i$ to the left of $\sigma^\smallfrown j$, whenever $i<j$. A \emph{path} $\pi\in \omega^{\omega}$ through a tree $\T\subseteq \omega^{<\omega}$ is an infinite sequence so that $\sigma  \in  \T$ for every $\sigma\prec \pi$. The set of paths through a tree $\T$ is denoted by $[\T]$. 
If $\T$ contains strings of arbitrary length, then $\T$ has \emph{infinite height}. Note that there are trees of infinite height which have no path, e.g., $\T=\{ n^\smallfrown \sigma : |\sigma|\leq n\}$. 

A tree $\T$ is \emph{finitely branching} if whenever $\sigma\in \T$ there are only finitely many numbers $i$ so that $\sigma^\smallfrown i\in \T$. $\T$ is \emph{computably finitely branching} if there is a computable function $g:\omega^{<\omega}\to\omega$ so that, for every $\sigma\in \T$, we  have $|\{i : \sigma^\smallfrown i\in \T\}| = g(\sigma)$.  K\"onig's lemma states that a finitely branching tree of infinite height has a path.


\subsubsection{Enumerating the computable functions and c.e.\ sets}
We let $(\varphi_e)_{e\in\mathbb{N}}$ denote a computable enumeration of all partial computable functions. 
We fix a single computable bijection $g:\omega^{<\omega}\rightarrow \omega$ and let $\hat{\phi}_e=\phi_e\circ g$. We say $e$ is a \emph{computable index for a tree} $\T$ if $\hat{\phi}_e(\sigma)=1$ if $\sigma\in\T$ and $\hat{\phi}_e(\sigma)=0$ otherwise.  

Similarly, $(W_e)_{e\in\mathbb{N}}$ denotes the computable enumeration of all computably enumerable (c.e.) sets given by $W_e:=\range(\phi_e)$.
For a c.e.\ set $W$, we denote by $W[s]$ the set of its elements
enumerated within $s$ steps. Without loss of generality, we assume that, for all $e,s\in\omega$, $|W[s+1]\smallsetminus W[s]|\leq 1$ (i.e., at all stages at most one number enters $W_e$).  We say that $s>0$ is an \emph{expansionary stage}, if $|W[s]|>|W[s-1]|$.

\subsubsection{Complexity classes and completeness}
Hierarchies form a core concept of computability theory, as they group sets of numbers in classes of increasing complexity. The \emph{arithmetical hierarchy}  classifies those sets which are definable in the language of first-order arithmetic according to the logical complexity of their defining formulas: $A\subseteq \mathbb{N}$ is  $\Sigma^0_n$ if there is a computable relation $R\subseteq \mathbb{N}^{n+1}$ such that $x\in A$ iff $(\exists y_1\forall y_2\ldots Q y_n)(R(x,y_1,\ldots,y_n))$ holds,
where $Q$ is $\exists$ if $n$ is odd and $\forall$ if $n$ is even; $A$ is $\Pi^0_n$ if $\omega\smallsetminus A$ is $\Sigma^0_n$. The $\Sigma^0_1$ sets coincide with the c.e.\ sets.

The \emph{difference hierarchy} gives a finer classification of the arithmetical sets. For our purposes, it suffices to mention the following complexity classes: $A$ is  $d$-$\Sigma^0_n$, if $A=X\cap Y$, for some $X\in \Sigma^0_n$ and $Y\in \Pi^0_n$, (or, equivalently, $A$ is the difference of two $\Sigma^0_n$ sets);  $A$ is  $u$-$\Sigma^0_n$, if $A=X\cup Y$, for some $X\in \Sigma^0_n$ and $Y\in \Pi^0_n$ (or, equivalently, $\omega\smallsetminus {A}$ is $d$-$\Sigma^0_n$).

Finally, the \emph{analytical hierarchy} emerges by allowing second-order quantification. Here, we are interested only in the first level of the hierarchy: The $\Sigma^1_1$ sets are the subsets of $\mathbb{N}$ that
are definable in the language of second-order arithmetic
using a \emph{single} second-order existential quantifier ranging
over subsets of $\mathbb{N}$; the $\Pi^1_1$ sets are the complements
of $\Sigma^1_1$ sets.

Let $\Gamma$ be a complexity class (e.g., $\Gamma\in 
\{\Sigma_n^0, d\text{-}\Sigma^0_n, u\text{-}\Sigma^0_n, \Sigma^1_1, \Pi^1_1\}$). A set $V\subseteq \mathbb{N}$ is \emph{$\Gamma$-hard}, if for every $X\in\Gamma$
there is a computable function $f :\mathbb{N} \to \mathbb{N}$ so that $x\in X$
iff $f(x)\in V$. If $V$ is $\Gamma$-hard and belongs to $\Gamma$,
then it is \emph{$\Gamma$-complete}.


To gauge the complexity of computational problems, it is convenient to use established benchmark sets:
\begin{itemize}
\item The halting set $K:=\{n: n\in W_n\}$ is $\Sigma^0_1$-complete. 
\item The sets $\Fin:=\{n: |W_n|<\infty\}$ and $\Inf:=\{n: |W_n|=\infty\}$ are, respectively, $\Sigma^0_2$- and $\Pi^0_2$-complete.
\item $\mathbf{Path} = \{e : e $ is a computable index for a tree which has a path$\}$ is $\Sigma^1_1$-complete.
\end{itemize} 


Thus, checking if a given tree has a path is $\Sigma^1_1$-hard. On the other hand, K\"onig lemma's ensures that verifying whether a given computably finitely branching tree $\T$ has a path is only $\Pi^0_1$ hard.



\section{Computably finitary argumentation frameworks}\label{sec:comply finitary AFs}



We now fix an indexing of all computably finitary AFs. Such an indexing captures the notion of being a computable AF, but even more, that we can computably see that the AF is finitary.

\begin{definition}\label{indexing of computable AFs}
A number $e$ is a \emph{computable index for a finitary AF} $\F=(A_\F,R_\F)$ with $A_\F=\{a_n : n\in \omega\}$, if $(\forall n,m)(a_n\att a_m \Leftrightarrow n\in D_{\phi_e(m)})$.
%


If $e$ is a computable index for a finitary AF, we let $\F_e$ represent this AF. An AF $\F$ is \emph{computably finitary}, if it possesses a computable index.
\end{definition}



\begin{remark}
    We are considering computably finitary AFs, which is a more restrictive class than the class of computable AFs which happen to be finitary. Consider for example the AF $\F=(A_\F,R_\F)$ where $A_\F=\{a_n : n\in \nat\}$ and $a_n\att a_m$ if and only if the computation $\phi_m(m)$ converges at exactly step $n$. Though $\F$ is a finitary computable AF, there is an element attacking $a_m$ if and only if $m$ is in the Halting set, thus $\F$ is not computably finitary.
    
    One might ask if our results remain valid for the larger class of computable AFs which happen to be finitary. The division of complexity between the arithmetical and non-arithmetical decision problems remains unchanged; though some of the complexities may change by up to one layer within the arithmetical hierarchy, e.g., we show $\cred_\stb^\icf$ is $\Pi^0_2$, whereas on the collection of computable AFs which happen to be finitary, it may be as complicated as $\Pi^0_3$. 
\end{remark}

In analogy with the case of computable AFs, we now present the computational problems naturally associated with  computably finitary AFs as subsets of $\omega$:

\begin{definition}\label{comp problem inf frameworks}
For a semantics $\sigma$:
\begin{enumerate}
\item $\cred_\sigma^{\icf}:=\{\langle e, n\rangle :  (\exists S\in \sigma(\F_e))(a_n\in S) \}$;
\item $\skep_\sigma^\icf:=\{\langle e, n\rangle :  (\forall S\in \sigma(\F_e))(a_n\in S) \}$;
\item $\ex_\sigma^\icf:=\{e :  (\exists S\subseteq A_{\F_e}))(S\in\sigma(\F_e)) \}$;
\item $\nemp_\sigma^\icf:=\{e :  (\exists S\in  \sigma(\F_e))(S\neq \emptyset)\}$;
\item $\uni_\sigma^\icf:=\{e :  (\exists! S\subseteq A_{\F_e})(S\in\sigma(\F_e))\}$.
\end{enumerate}

\end{definition}

For items $\cred_\sigma^\icf$ and $\skep_\sigma^\icf$ above,  we also consider
their restrictions to a specific computably finitary AF $\F_e$. That is,
we define 
\begin{itemize}
\item $\cred^\icf_\sigma(\F_e)=\{n:\langle e, n\rangle \in \cred^\icf_\sigma\}$;
\item $\skep^\icf_\sigma(\F_e)=\{n:\langle e, n\rangle \in \skep^\icf_\sigma\}$
\end{itemize}

For the sake of exposition, we abuse notation and identify arguments or AFs with their indices. For example, we may write $\F\in\ex^\icf_\sigma$ to mean that $\F$ is a computably finitary AF having a $\sigma$-semantics; or, we may write $a_n\in \skep^\icf_\sigma(\F)$ instead of $n\in \skep^\icf_\sigma(\F)$.

The main goal of this paper is to investigate the complexity of the problems listed in Definition \ref{comp problem inf frameworks}. Yet, in pursuing this goal, one encounters a basic issue: namely, part of the complexity of each problem $\mathcal{P}^\icf_\sigma$ comes from the complexity of determining whether an index $e$ is in fact an index for a computably finitary AF, in particular if $\phi_e$ is a total function.
A natural way to overcome this issue and be more faithful to the complexity of a given decision problem $\mathcal{P}^\icf_\sigma$ is by understanding our complexity classes on this set:

\begin{definition}
Let $X$ denote some subset of $\nat$.
For each of our computational problems $\mathcal{P}_\sigma^\icf$:
\begin{itemize}
\item $\mathcal{P}^\icf_\sigma$ is \emph{$\Gamma$ within $X$} if $\mathcal{P}^\icf_\sigma=R\cap X$ for some $R\in \Gamma$;


\item $\mathcal{P}^\icf_\sigma$ is \emph{$\Gamma$-complete within $X$} if 
\begin{itemize}
    \item $\mathcal{P}^\icf_\sigma$ is $\Gamma$ within $X$; 
    \item Whenever $S\in\Gamma$, there exists a computable $f:\omega \to X$ so that $x\in S \Leftrightarrow f(x)\in \mathcal{P}^\icf_\sigma$.
\end{itemize}
\end{itemize}
\end{definition} 

The last definition grounds the following important remark:

\begin{remark}\label{rem:inTot}
All  complexity results are to be understood under assumption that \emph{we are reasoning within $\tot$}, the set of indices for total computable functions. Note that $e$ is a computable index for a finitary AF if and only if $\phi_e$ is a total function. That is, for $\mathcal{P}\in \{\ex,\nemp,\uni\}$, we examine complexity of $\mathcal{P}^\icf_\sigma$ in $\tot$, and for $\mathcal{P}\in \{\cred,\skep\}$, we examine complexity of  $\mathcal{P}^\icf_\sigma$ in $\tot\times \nat$. Since $\tot$ itself has complexity $\Pi^0_2$, this has no effect on our results showing $\Gamma$-completeness for any $\Gamma$ closed under conjunctions with $\Pi^0_2$ sets, e.g., $\Sigma^1_1$-completeness. For the others, e.g., the $\Sigma^0_1$-completeness in $\tot$ of $\skep^\icf_{\stb}$, this means that given an index $e$ for a computably finitary AF $\F_e$, and an $n\in \omega$ it is $\Sigma^0_1$ to see that $a_n$ is in every stable extension in $\F_e$. It would be misleading to say that this problem is $\Pi^0_2$-hard, since the complexity at that level is the complexity of determining whether or not $e$ is an index for a computably finitary AF, which does not accurately reflect the complexity of the problem of credulous acceptance for the admissible semantics.
\end{remark}

\section{Results}\label{sec:results}

The main results of this paper are gathered in the following tables. 

\begin{table}[ht]
\caption{Computational problems for computably finitary AFs. $\mathcal{C}$-c denotes completeness for the class $\mathcal{C}$. 
The NE and Exists columns are joined for the semantics which require the extensions to be infinite, thus non-empty. Highlighted cells represent problems where the complexity is far simpler than the non-finitary case.}
	\label{finitaryTable}
	\centering
	\begin{tabular}{|l|l|l|l|l|l|}
		\hline
		$\sigma$ & $\cred^\icf_\sigma$ & $\skep^\icf_\sigma$ & $\ex^\icf_\sigma$ & $\nemp_\sigma^{\icf}$ & $\uni^\icf_\sigma$  \\
        
            \hline
		$\cf$ & computable 
        & trivial 
        & trivial 
        & $\Sigma^0_1$-c 
        & $\Pi^0_1$-c 
        \\
		\hline

		$\na$ 
        & computable 
        & $\Pi^0_1$-c 
        & trivial 
        & $\Sigma^0_1$-c 
        & $\Pi^0_1$-c 
        \\
		\hline

		$\ad$ 
        & \cellcolor{myColor} $\Pi^0_1$-c 
        & trivial  & trivial 
        &\cellcolor{myColor} $\Sigma^0_2$-c 
        &\cellcolor{myColor} $\Pi^0_2$-c 
        \\
		\hline


		$\stb$ 
        & \cellcolor{myColor} $\Pi^0_1$-c 
        & \cellcolor{myColor} $\Sigma^0_1$-c 
        & \cellcolor{myColor}$\Pi^0_1$-c 
        & \cellcolor{myColor}$\Pi^0_1$-c 
        & \cellcolor{myColor}$\Pi^0_2$-c  
        \\
		\hline	
        
		$\co$ 
        & \cellcolor{myColor}$\Pi^0_1$-c 
        &\cellcolor{myColor}  $\Sigma^0_1$-c 
        & trivial 
        & \cellcolor{myColor}$\Sigma^0_2$-c 
        & \cellcolor{myColor}$\Pi^0_2$-c 
        \\
		\hline
        
        $\infcf$ 
        & $\Sigma^1_1$-c 
        & $\Pi^1_1$-c 
        & \multicolumn{2}{c|}{$\Sigma^1_1$-c 
        } 
        & trivial 
        \\
		\hline

		$\infna$ 
        & $\Sigma^1_1$-c 
        & $\Pi^1_1$-c 
        & \multicolumn{2}{c|}{$\Sigma^1_1$-c 
        } 
        & $\Pi^1_1$-c 
        \\
		\hline

            $\infad$ 
            &\cellcolor{myColor} u-$\Sigma^0_2$-c 
            &\cellcolor{myColor} d-$\Sigma^0_2$-c 
            &\multicolumn{2}{c|}{\cellcolor{myColor} u-$\Sigma^0_2$-c 
            } 
            & \cellcolor{myColor} $\Pi^0_3$-c 
            \\
		\hline
        
            $\infstb$ 
            & \cellcolor{myColor} $\Pi^0_2$-c 
            &\cellcolor{myColor}  $\Sigma^0_2$-c 
            & \multicolumn{2}{c|}{\cellcolor{myColor} $\Pi^0_2$-c 
            } 
            &\cellcolor{myColor}  $\Pi^0_3$-c 
            \\
		\hline
        
            $\infco$ 
            &\cellcolor{myColor}  u-$\Sigma^0_2$-c 
            & ? & \multicolumn{2}{c|}{\cellcolor{myColor} u-$\Sigma^0_2$-c 
            }  
            & ? \\
		\hline
	\end{tabular}
\end{table}

\begin{table}[ht]
\caption{This Table contains the results for infinite AFs without any assumption of finitarity. 
The results for the asterisked semantics are proved in this paper, and the other results can be found in \cite{andrews2024NMR} and \cite{TARK}.
}
	\label{infinitaryTable}
	\centering
	\begin{tabular}{|l|l|l|l|l|l|}
		\hline
		$\sigma$ & Cred$^\infty_\sigma$ & Skept$^\infty_\sigma$ & Exists$^\infty_\sigma$ & NE$_\sigma^{\infty}$ & Uni$^\infty_\sigma$  \\
		
		\hline
      	$\cf$ $ \ast$ & computable   & trivial  & trivial & $\Sigma^0_1$-c  & $\Pi^0_1$-c  \\
    		\hline

    		$\na$ $ \ast$ & computable & $\Pi^0_1$  & trivial & $\Sigma^0_1$-c  & $\Pi^0_1$-c  \\
		\hline
		$\ad$ & $\Sigma^1_1$-c & trivial  & trivial & $\Sigma^1_1$-c & $\Pi^1_1$-c  \\
		\hline


		$\stb$ & $\Sigma^1_1$-c  &  $\Pi^1_1$-c  & $\Sigma^1_1$-c  & $\Sigma^1_1$-c  & $\Pi^1_1$-c    \\
		\hline	
		$\co$ & $\Sigma^1_1$-c  & $\Pi^1_1$-c  & trivial  & $\Sigma^1_1$-c  &$\Pi^1_1$-c    \\
		\hline
           $\infcf$ $ \ast$& $\Sigma^1_1$-c  & $\Pi^1_1$-c  & \multicolumn{2}{c|}{$\Sigma^1_1$-c } & trivial  \\
        		\hline
        
		$\infna$ $ \ast$& $\Sigma^1_1-$c  & $\Pi^1_1$-c   & \multicolumn{2}{c|}{$\Sigma^1_1$-c } & $\Pi^1_1$-c \\
		\hline
            $\infad$ & $\Sigma^1_1$-c   & $\Pi^1_1$-c  & \multicolumn{2}{c|}{$\Sigma^1_1$-c  } &$\Pi^1_1$-c    \\
		\hline
            $\infstb$ & $\Sigma^1_1$-c   & $\Pi^1_1$-c  & \multicolumn{2}{c|}{$\Sigma^1_1$-c} &$\Pi^1_1$-c   \\
		\hline
            $\infco$ & $\Sigma^1_1$-c   & $\Pi^1_1$-c  & \multicolumn{2}{c|}{$\Sigma^1_1$-c}  &$\Pi^1_1$-c \\
		\hline
	\end{tabular}

\end{table}

In the next sections, we will present the results collected in Table \ref{finitaryTable} and the new results of Table \ref{infinitaryTable}. 
We highlight the distinctions between entries in Table \ref{finitaryTable} and Table \ref{infinitaryTable} in the rows corresponding to $\ad,\stb,\co,\infad,\infstb,\infco$, where there is a stark difference in complexity. Namely, for general infinite AFs, these problems are non-arithmetical, whereas for finitary AFs, they are at very low levels of the arithmetical hierarchy. 

\begin{remark}\label{computable approximations}
    The complexity-classes at very low levels of the arithmetical hierarchy appearing in Table \ref{finitaryTable} correspond to approximation-strategies for the corresponding decision problems.
    When a decision problem $X$ is $\Sigma^0_1$, such as $\skep^\icf_\stb$, there is a computable function $f:\omega^2\rightarrow \{0,1\}$ so that $f(x,s)\leq f(x,s+1)$ so that the characteristic function of $X$ equals $\lim_{s\to\infty} f(x,s)$. Equivalently, the decision problem is computably enumerable.
    When a decision problem $X$ is $\Sigma^0_2$, such as $\skep^\icf_\infstb$, there is a computable function $f:\omega^2\rightarrow \{0,1\}$ so that the characteristic function of $X$ equals $\liminf_{s\to\infty} f(x,s)$. The other classes have similarly defined approximation algorithms. 

    We note that solving the decision problem via a limiting procedure from finite computations is as good as one might hope for. Suppose a logical reasoner, which is being modeled by the decision procedure, were processing the infinitely much information in an infinite AF and trying to decide on which arguments to accept. Then at any finite time, only finitely much information can possibly have been processed, so the reasoner has simply not seen enough information to come to a definitive conclusion. The best we may hope for is that as time moves forward, and the reasoner acquires all pertinent information, their finite-times conclusions should approximate the correct conclusions.

    In this way, we are able at finite time to make approximations to the solutions of, say $\cred^\icf_{\infstb}$ so that the $\limsup$ of the approximations is correct. This is in contrast to the non-finitary setting where the $\Sigma^1_1$-completeness of $\cred^\infty_\infstb$ shows that no finite-time computation can in any way approximate the solution to this problem.
\end{remark}

\section{The classic semantics}\label{sec: cfna}



The complexity of the computational problems for the conflict-free and naive semantics in the infinite setting were not covered in \cite{andrews2024NMR}. In this section,  we see that their complexity remains unaffected by the assumption of the AF being finitary: for $\mathcal{P}\in\{\cred,\skep,\ex,\nemp,\uni\}$,  
both  $\mathcal{P}^\icf_\sigma$ and $\mathcal{P}^\infty_\sigma$ have the same complexity.
Since the empty set is conflict-free,  $\skep_\cf^\infty$,  $\ex_\cf^\infty$, and $\ex_\na^\infty$ are trivial. The next couple theorems, proved in Sections \ref{supplement: cf and na}, analyze the remaining complexities associated with the conflict-free and the naive semantics. 

\begin{restatable}{theorem}{allcf}
\label{all cf}
$(i)$ $\cred_\cf^\icf$ is computable; $(ii)$ $\nemp_\cf^\icf$ is $\Sigma^0_1$-complete; $(iii)$ $\uni_\cf^\icf$ is $\Pi^0_1$-complete.
\end{restatable}

\begin{restatable}{theorem}{allnaive}
\label{all naive} $(i)$ $\cred_\na^\icf$ is computable; $(ii)$ $\skep_\na^\icf$ is $\Pi^0_1$-complete; $(iii)$ $\nemp_\na^\icf$ is $\Sigma^0_1$-complete; $(iv)$ $\uni_\na^\icf$ is $\Pi^0_1$-complete.
\end{restatable}

 In sharp contrast to the conflict-free and naive cases, we find  that, for the admissible, stable, and complete semantics, the finitary case is significantly  simpler than the general case. Specifically, the (non-trivial) computational problems associated with such semantics drop from being analytical in the non-finitary setting to being in low arithmetical levels for finitary AFs.

The main technique we use for providing upper bounds for the computational problems involves relating the set of accepted extensions
in $\F$ with the set of paths through a tree via a suitable encoding, summarized in the following theorem (see Section \ref{appendix: encoding extensions into trees} and \ref{supplement: ad-based} for the proofs).

\begin{restatable}{theorem}{thetreetheorem}\label{encoding extensions into trees}
    Let $\F$ be an argumentation framework, $D,E\subseteq A_\F$ and $\sigma\in \{\ad, \co, \stb\}$. Then there is a tree $\T^\F_{\sigma+D-E}$ so that the $\sigma$ extensions of $\F$ which contain $D$ and are disjoint from $E$ are in bijection with the paths through $\T^\F_{\sigma+D-E}$.

    Further, the tree $\T^\F_{\sigma+D-E}$ is uniformly computable from $\F$, $D$, and $E$. Finally, if $\F$ is a computably finitary AF, then the tree $\T^\F_{\sigma+D-E}$ is computably finitely branching.
\end{restatable}






\begin{restatable}{theorem}{easyuppers}
\label{easy uppers}
For $\sigma\in\{\ad,\co,\stb\}$, $\cred^\icf_\sigma$ is $\Pi^0_1$-complete. $\ex^\icf_\stb$ is also $\Pi^0_1$-complete.
\end{restatable}

\begin{restatable}{theorem}{OTHupper}
\label{NE upper in paper}\label{unico upper in paper}\label{unistb upper}
   $(i)$ $\nemp^\icf_\stb$ is $\Pi^0_1$-complete; $(ii)$ $\nemp^\icf_\ad$ and $\nemp^\icf_\co$ are each $\Sigma^0_2$-complete; $(iii)$ $\skep^\icf_\co$ and $\skep^\icf_\stb$ are $\Sigma^0_1$-complete; $(iv)$ For $\sigma\in \{\ad,\co,\stb\}$, $\uni^\icf_\sigma$ are each $\Pi^0_2$-complete.
\end{restatable}

\section{Infinite conflict-free and infinite naive\label{sec:infcf infna}}

We now turn to the analysis of $\infcf$ and $\infna$. With the exception of $\uni_\infcf^\icf$, the computational problems associated with these semantics turn out to be  maximally hard (see Section \ref{supplement: sect 6} for the proofs).

\begin{restatable}{theorem}{credexinfcfinfna}
    \label{cred ex infcf infna}\label{cred infcf infna}
  $(i)$  $\cred^\icf_\infcf$ and $\ex^\icf_\infcf$ are $\Sigma^1_1$-complete; $(ii)$  $\cred^\icf_\infna$ and $\ex^\icf_\infna$ are $\Sigma^1_1$-complete;     $(iii)$ $\skep^\icf_\infcf$ and $\skep^\icf_\infna$ are $\Pi^1_1$-complete. 
\end{restatable}

\begin{restatable}{theorem}{uniinfcf}
\label{uni infcf}
 $(i)$  $\uni^\icf_\infcf$ is trivial; $(ii)$  $\uni^\icf_\infna$ is $\Pi^1_1$-complete.
\end{restatable}

\begin{proposition}\label{infcf, infna infinitary}
    For each $\mathcal{P}\in \{\cred,\skep,\ex,\nemp,\uni\}$ and $\sigma\in \{\infcf,\infna\}$, $\mathcal{P}^\infty_\sigma$ has the same complexity as $\mathcal{P}^\icf_\sigma$.
\end{proposition}
\begin{proof}
    In this section, all the proofs of the upper-bounds work for $\mathcal{P}^\infty_\sigma$, whereas the complexity of $\mathcal{P}^\icf_\sigma$ is always a lower-bound for the complexity of $\mathcal{P}^\infty_\sigma$.
\end{proof}

\section{Infinite admissible\label{sec:infad}}

We now turn to considering the semantics of infinite admissible extensions. 
There is no reason to expect a simple characterization of any of the computational problems for $\infad$, $\infstb$, or $\infco$ in the finitary setting. In the non-finitary setting they are maximally hard (see Table \ref{infinitaryTable}), and they could end up being $\Sigma^1_1$- or $\Pi^1_1$-complete in the finitary setting as well, as happened in the case of $\infcf$ or $\infna$ (Theorems \ref{cred ex infcf infna} and \ref{uni infcf}). 

Yet, in this section we will show a surprisingly nice characterization for $\infad$ in the finitary setting.
This characterization relies on the following theorem which is  a purely combinatorial result about finitary AFs. We believe that this result will be of interest beyond its immediate applicability as a tool to understand the interplay between our computational problems. For example, one immediate consequence of Theorem \ref{Infinitely many finie implies one infinite} is the following unexpected relation between two distinct semantics: For $\F$ finitary, $|\cred_\ad(\F)|=\infty \Leftrightarrow \F \in \ex_\infad$.

\begin{theorem}\label{Infinitely many finie implies one infinite}
    Let $\F=(A_\F,R_\F)$ be a finitary argumentation framework. Fix $D\subseteq A_\F$ and $E\subseteq A_\F$. Let $Y$ be the set of arguments $b\in A_\F$ so that there exists some admissible extension $X_b$ so that $D\cup \{b\}\subseteq X_b$ and $E\cap X_b = \emptyset$.
    
    Then $D$ is contained in an infinite admissible extension which is disjoint from $E$ iff the set $Y$ is infinite.
\end{theorem}

\begin{proof}
    Observe that if $S$ is an infinite admissible extension of $\F$ containing $D$ and disjoint from $E$, then $S\subseteq Y$, so $Y$ is infinite.

    Conversely, suppose that $Y$ is infinite. For each $b\in Y$, fix an admissible extension $X_b$ with $D\cup \{b\}\subseteq X_b$ and $E\cap X_b = \emptyset$. If any $X_b$ is infinite, we are done, so we may assume each $X_b$ is finite. Thus, we can choose an infinite subset $Y'$ of $Y$ so that $X_b$ is distinct for each $b\in Y'$. We define an edge relation $E$ on the collection of arguments $Y'$. We say $b$ and $b'$ are $E$-related if there is some element of $X_b$ which attacks some element of $X_{b'}$. Observe by the admissibility of $X_b$ and $X_{b'}$ that this is a symmetric relation. By the Infinite Ramsey theorem (see e.g., \cite[Theorem 5]{graham1980ramsey}), there is either an infinite clique or an infinite anti-clique in this graph. If there is an infinite anti-clique $Z$, then $\bigcup_{b\in Z} X_b$ is an infinite admissible set. So, we may suppose that there is an infinite clique $C$.
    
    We proceed by recursion to define a sequence of sets $S_i$. We begin with $S_0=D$. We will maintain the inductive hypothesis ($\IH$) that the set $C_i:=\{b\in C : S_i\subseteq X_b\}$ is infinite.

    \begin{lemma}\label{Add An Element}
        For any $S_i$, there is always some argument $x\notin S_i$ so that letting $S_{i+1}=S_i\cup \{x\}$ preserves  \emph{$\IH$}.
    \end{lemma}
    \begin{proof}
        Fix any $b$ so that $S_i\subseteq X_b$. By the pigeonhole principle, there is one $z\in X_b$ so that there are infinitely many $b'\in C_i$ so that some element of $X_{b'}$ attacks $z$. Since $\F$ is finitary, there must be a single element $x$ which attacks $z$ and is contained in infinitely many $X_{b'}$ for $b'\in C_i$. Then $S_i\cup \{x\}$ is contained in infinitely many $X_{b'}$ with $b'\in C$, so  $\IH$ is preserved. Note that $x\notin S_i$ since $S_i\cup \{z\}\subseteq X_b$ and $X_b$ is conflict-free.
    \end{proof}

    \begin{lemma}\label{Make It Admissible}
        For any $S_i$, $y\in S_i$ and $z$ any argument attacking $y$, there is some $x$ so that $x$ attacks $z$ and letting $S_{i+1}=S_i\cup \{x\}$ preserves  \emph{$\IH$}.
    \end{lemma}
    \begin{proof}
        For each $b\in C_i$, there must be some $x_b\in X_b$ so that $x_b$ attacks $z$, since $X_b$ is admissible and contains $y$. Since there are only finitely many elements attacking $z$, the pigeonhole principle implies there is a single $x$ so that there are infinitely many $b$ with $x_b=x$. Then $S_{i+1}=S_i\cup \{x\}$ maintains  $\IH$.
    \end{proof}

    We now proceed to build the sequence of sets $S_i$ starting with $S_0=D$. At odd steps, we let $S_{i+1}$ be as guaranteed by Lemma \ref{Add An Element}. By doing this at each odd stage, we guarantee that $S:=\bigcup_i S_i$ is infinite. At even steps, we take the least $\langle y,z\rangle$ so that $y\in S_i$ and $z$ attacks $y$ and apply Lemma \ref{Make It Admissible} to define $S_{i+1}$. If there is no such pair $(y,z)$, then we simply let $S_{i+1}=S_i$. Since we do this infinitely often, for any element of $y\in S$ which is attacked by an element $z$, we place an element $x\in S$ which attacks $z$. Finally, note that $S$ is conflict-free since each $S_i$ is conflict-free by the $\IH$ and the fact that each $X_b$ is conflict-free. Thus $S$ is an infinite admissible set containing $D$. Finally, since each $S_i$ is contained in infinitely many $X_b$ and each $X_b$ is disjoint from $E$, $S$ is disjoint from $E$.
\end{proof}

We use Theorem \ref{Infinitely many finie implies one infinite} to give an upper bound for the complexities of the computational problems regarding the semantics $\infad$. Using Theorem \ref{Infinitely many finie implies one infinite}, to check if there is an infinite admissible extension, we need only check if, for every $n$, there is an admissible extension of size $\geq n$. This is a $\Pi^0_3$ description of $\ex^\icf_\infad$, though we are able to push this upper bound down to the level of $u$-$\Sigma^0_2$, which we show is sharp in Theorem \ref{infad lower}.

\begin{theorem}\label{infad upper}
$\ex^\icf_\infad$ and $\cred^\icf_\infad$ are each $u$-$\Sigma^0_2$.
$\skep^\icf_\infad$ is $d$-$\Sigma^0_2$.
\end{theorem}
\begin{proof}
Recall that $u$-$\Sigma^0_2$ sets are the union of a $\Sigma^0_2$ and $\Pi^0_2$ set, while $d$-$\Sigma^0_2$ sets---being the complements of $u$-$\Sigma^0_2$ sets---are the intersection of a $\Sigma^0_2$ and $\Pi^0_2$ set.

 We first focus on $\ex^\icf_\infad$.
  Let $\F$ be a computably finitary AF with $A_\F=\{a_n: n\in\omega\}$. Recall that $D_k$ denotes the finite set with canonical index $k$. Let $\hat{D}_k = \{a_i : i\in D_k\}$. Consider the following condition: 
\begin{multline}
\tag{$\mathcal{E}$}
(\exists n\in\omega)\left( (\exists k\in \omega)\left(|D_k|=n  \, \&  \, [\T^\F_{\ad+\hat{D}_k-\emptyset}]\neq \emptyset\right)\, \& \right.\\ \left. \, (\forall  k)\left(|D_k|\geq n \Rightarrow \hat{D}_k \notin \ad(\F)\right) \right)
\end{multline}
    The first conjunct simply says there exists an admissible extension of size $\geq n$; the second conjunct says that no finite admissible extension in $\F$ has size $\geq n$. By K\"onig's Lemma, the existence of a path through $\T^\F_{\ad+\hat{D}_k-\emptyset}$ is a $\Pi^0_1$ condition (as the tree is computably finitely branching). Checking whether a finite set is admissible is  computable, thus the overall complexity of   $\mathcal E$ is $\Sigma^0_2$.

    Next, let $\mathcal A$ be the following condition, which says that $\F$ has arbitrarily large finite admissible extensions:
\[
\tag{$\mathcal{A}$}(\forall n\in \omega)(\exists k\in \omega)\left(|D_k|>n \; \& \; \hat{D}_k\in \ad(\F)\right)
\]  
Observe that $\mathcal{A}$  is a  $\Pi^0_2$ condition, since verifying the admissibility of a finite set is computable.

    If  $\mathcal E$ holds, then $\F$ has an infinite admissible extension, since some admissible extension of size $\geq n$ exists and this cannot be finite. Similarly, if condition $\mathcal{A}$ holds, then there are infinitely many elements which are contained in an admissible extension, so there is an infinite admissible extension by Theorem \ref{Infinitely many finie implies one infinite} applied with $D=E=\emptyset$.

    Next, suppose that there is an infinite admissible extension in $\F$. If there are arbitrarily large finite admissible extensions, then condition $\mathcal A$ holds. If not, then let $n$ be an upper bound to the size of all finite admissible extensions. Observe that $n$ witnesses that condition $\mathcal E$ holds. Thus, the collection of computably finitary AFs with an infinite admissible extension is exactly the union of those $\F$  satisfying $\mathcal{E}$ and those satisfying $\mathcal{A}$, which shows that $\ex^\icf_\infad$ is $u$-$\Sigma^0_2$ 

  The above argument still holds if we replace the notion of admissible extension by that of ``admissible extension containing $a$'': indeed, it suffices to use  $D=\{a\}$ and $E=\emptyset$  in the application of Theorem \ref{Infinitely many finie implies one infinite}. This shows that $\cred^\icf_\infad$ is a union of a $\Sigma^0_2$ and a $\Pi^0_2$ set.

 Finally, by using $D=\emptyset$ and $E=\{a\}$ in the application of Theorem \ref{Infinitely many finie implies one infinite}, we obtain that the above argument also holds if we replace the notion of admissible extension by ``admissible extension omitting $a$''.  Thus, the existence of an admissible extension which does not include $a$ is  $u$-$\Sigma^0_2$. Since  $\skep^\icf_\infad$ is defined by the negation of this condition, it is $d$-$\Sigma^0_2$.  
\end{proof}

\begin{theorem}\label{infad lower}
$\ex^\icf_\infad$ and $\cred^\icf_\infad$ are each $u$-$\Sigma^0_2$-hard. $\skep^\icf_\infad$ is $d$-$\Sigma^0_2$-hard.
\end{theorem}
\begin{proof}
     We begin by first giving constructions proving the $\Sigma^0_2$-hardness and $\Pi^0_2$-hardness  of $\ex^\icf_\infad$; then we combine these constructions to show $u$-$\Sigma^0_2$-hardness. 

Recall that the set $\Fin$ (the collection of indices of finite c.e.\ sets) is $\Sigma^0_2$-complete.  We will construct a uniformly computable sequence of computably finitary AFs $(\G_n)_{n\in\omega}$ so that $\G_n$ has an infinite admissible extension iff $n\in \Fin$.  
    We define $\G_n=(A_{\G_n},R_{\G_n})$ as follows (see Figure \ref{figure: exist sigma2-hard}). Let $A_{\G_n}$ be $\{a_k : k\in \omega\} \cup  \{b_k : \text{$k$ is an expansionary stage for $W_n$}\}$. Let $R_\F$ contain exactly the following attack relations: For all $k$,
\begin{itemize}
\item $a_{k+1}\att a_k$;
\item $b_k\att a_{2k}$ and $b_k\att a_{2k+1}$;
\item $b_k\att b_k$.
\end{itemize}

\begin{figure}\centering
\scalebox{0.5}{
\begin{tikzpicture}[->=Stealth]

    \node[circle, draw, 
    ] (a0) {$a_{2s_0}$};
    \node[draw=none] (ellipsis1)[left=5mm of a0] {$\cdots$};
    \node[circle, draw, 
    ] (a1) [right=of a0] {$a_{2s_0+1}$};
    \node[circle, draw, 
    ] (a2) [right=of a1] {$a_{2s_1}$};
    \node[circle, draw, 
    ] (a3) [right=of a2] {$a_{2s_1+1}$};
    \node[circle, draw, 
    ] (a4) [right=of a3] {$a_{2s_2}$};
    \node[circle, draw, 
    ] (a5) [right=of a4] {$a_{2s_2+1}$};
    \node[circle, draw, 
    ] (a6) [right=of a5] {$a_{2s_3}$};
    \node[circle, draw, 
    ] (a7) [right=of a6] {$a_{2s_3+1}$};
    \node[draw=none] (ellipsis2)[right=1mm of a7] {$\cdots$};

    \node[circle, draw, 
    ] (b0) 
    [below=of a0] {$b_{s_0}$};
    \node[circle, draw, 
    ] (b1) [below=of a2] {$b_{s_1}$};
    \node[circle, draw, 
    ] (b2) [below=of a4] {$b_{s_2}$};
    \node[circle, draw, 
    ] (b3) [below=of a6] {$b_{s_3}$};

\draw[->] (a1) edge (a0);
\draw[->] (a2) edge (a1) [dashed];
\draw[->] (a3) edge (a2);
\draw[->] (a4) edge (a3) [dashed];
\draw[->] (a5) edge (a4);
\draw[->] (a6) edge (a5) [dashed];
\draw[->] (a7) edge (a6);

\draw[->] (b0) edge[loop below] (b0);
\draw[->] (b1) edge[loop below] (b1);
\draw[->] (b2) edge[loop below] (b2);
\draw[->] (b3) edge[loop below] (b3);

\path[every node/.style={font=\sffamily\small}]

    (b0) edge  (a0)
    (b0) edge[out=80, in=-90] (a1)
    (b1) edge (a2)
    (b1) edge[out=80, in=-90] (a3)
    (b2) edge  (a4)
    (b2) edge[out=80, in=-90] (a5)
    (b3) edge  (a6)
    (b3) edge[out=80, in=-90] (a7)
    ;
\end{tikzpicture}
}

\caption{A fragment of $\G_n$ where $W_n$ has expansionary stages $s_0<s_1<s_2<s_3$.}
\label{figure: exist sigma2-hard}

\end{figure}
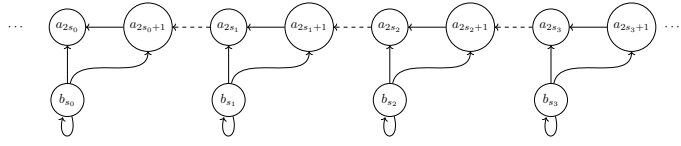
If $W_n$ is finite, let $t$ be its last expansionary stage and let $S=\{a_{2m} : m> t\}$. Observe that $S$ is an infinite admissible set.   

On the other hand, if there are infinitely many expansionary stages, then we argue that $\G_n$ has no non-empty admissible extension. Towards a contradiction, suppose that $S\neq \emptyset$ is in $\ad(\G_n)$. Note that, for all $k$, $b_k\notin S$ as $b_k\att b_k$. Thus, if $S$ is a non-empty admissible extension, then there is some $a_j\in S$. Due to the attacks $a_{k+1}\att a_k$, we must have $a_{j+2m}\in S$ for every natural number $m$. Fix an expansionary stage $t>k$. By construction, $b_t$ attacks both $a_{2t}$ and $a_{2t+1}$, one of which (depending on the parity of $j$) is in $S$. This attack is undefended, contradicting the  admissibility of $S$. Thus $n\in \Fin$ iff $\G_n\in \ex^\icf_\infad$, showing that $\ex^\icf_\infad$ is $\Sigma^0_2$-hard.

%
%

    Next, we will prove that $\ex^\icf_\infad$ is $\Pi^0_2$-hard by reducing the set $\Inf$ to it. Specifically, we construct a computable sequence of computably finitary AFs $(\H_n)_{n\in\omega}$ so that  $\H_n$ has an infinite admissible extension iff $n\in \Inf$. 
    
    We call the following AF a \emph{star}: $\C=(A_\C,R_\C)$ where $A_\C=\{a_n : n\in \omega\}$ and $a_n\att a_m$ iff $n=0$ and $m\neq 0$. We call $a_0$ the center of the star. For each $n$, we let $\H_n$ be the disjoint union of $|W_n|+1$ stars.
    

Note that the admissible subsets of $\H_n$ are exactly the subsets of the set of centers of the stars in $\H_n$. Thus $\H_n$ has an infinite admissible extension iff $W_n$ is infinite.



    Finally, let $X\in\Sigma^0_2$, $Y\in\Pi^0_2$, and  $n\in \omega$. We can construct sequences $\G_n$ so that $\G_n\in \ex^\icf_\infad$ iff $n\in X$ and $\H_n\in \ex^\icf_\infad$ iff $n\in Y$. We then construct the AF $\I_n$ which is the disjoint union of $\G_n$ and $\H_n$. If $n\in X$, then there is an infinite admissible subset of $\G_n$, which is admissible in $\I_n$. If $n\in Y$, then there is an infinite admissible subset of $\H_n$, which is admissible in $\I_n$. If neither, then neither $\G_n$ nor $\H_n$ have infinite admissible extensions, thus the disjoint union $\I_n$ does not have an infinite admissible extension. 

    To see that $\cred^\icf_\infad$ is hard for $u$-$\Sigma^0_2$ sets, consider $\I_n$ and an argument $a$ which is the center of a star in $\H_n$. This element can be added to any admissible extension, if one exists, since $a$ is unattacked. Thus, $a$ is credulously accepted for the  $\infad$-semantics iff there exists an infinite admissible extension iff $n\in X\cup Y$. Note that we ensured that $\H_n$ always has at least one star, so we can find the element $a$ that we need here. 

Similarly, if we take an argument $a$ which is not the center of a star, but rather outside the center of a star, then $a$ won't be in any admissible extension. Thus, $a$ is skeptically accepted iff there is no infinite admissible extension. Since we know that $\ex^\icf_\infad$ is $u$-$\Sigma^0_2$-hard, this shows that $\skep^\icf_\infad$ is hard for the complementary class, i.e.,  $d$-$\Sigma^0_2$.
\end{proof}


The proof of the next theorem is given in Section \ref{supplement: sec 7}.

\begin{restatable}{theorem}{uniinfad}
    \label{uni infad}
    $\uni^\icf_\infad$ is $\Pi^0_3$-complete.
\end{restatable}

We are not prepared to establish the complexity of $\skep^\icf_\infco$ or $\uni^\icf_\infco$. Specifically, for the complete semantics, the method of Theorem \ref{Infinitely many finie implies one infinite} in the case where $Z$ is an infinite anti-clique does not work. That is, if $X_b$ are each complete extensions which do not attack each other and do not contain a fixed element $a$, then $\bigcup_b X_b$ is admissible, but it may not be contained in a complete extension which omits $a$. Yet, it suffices to recall that  any admissible extension is contained in a complete extension to determine the complexity of $\ex^\icf_\infco$ and $\cred^\icf_\infco$:
\begin{theorem}\label{trivial infco}
    $\ex^\icf_\infco$ is exactly the same as $\ex^\icf_\infad$.
    $\cred^\icf_\infco$ is exactly the same as $\cred^\icf_\infad$.
\end{theorem}


\section{Infinite stable\label{sec:infstb}}

As in the case of the infinite admissible semantics, our analysis of the infinite stable semantics relies upon the following theorem which is a purely combinatorial result about finitary AFs, which we think is of independent interest. This theorem, proved in Section \ref{supplement: sec 8} together with Theorems \ref{infstb lower} and \ref{uniinfstb}, is the analogue for the stable semantics of Theorem \ref{Infinitely many finie implies one infinite}.

\begin{restatable}{theorem}{combinatoricsStable}
    \label{Infinitely many finite implies one infinite - stable}
    Let $\F=(A_\F,R_\F)$ be a finitary argumentation framework. Fix $D\subset A_\F$ and $E\subseteq A_\F$. Let $Y$ be the set of arguments $b\in A_\F$ so that there exists some stable extension $X_b$ so that $D\cup \{b\}\subseteq X_b$ and $E\cap X_b = \emptyset$.
    
    Then $D$ is contained in an infinite stable extension which is disjoint from $E$ iff the set $Y$ is infinite.
\end{restatable}

Despite the similarity of Theorem \ref{Infinitely many finie implies one infinite} and Theorem \ref{Infinitely many finite implies one infinite - stable}, we will find different complexities for the problems with $\infad$ and $\infstb$. This is due to two differences: Firstly, determining if a finite set is stable is $\Pi^0_1$ whereas it is computable to determine if it is admissible. Secondly, if $X$ is a finite stable extension, then there can be no other stable extension containing $X$. This means that in the tree $\T^\F_{\stb+\emptyset-\emptyset}$, once a node determines that $X\subseteq S$, there can be no infinite stable extension represented by a path containing this node.

\begin{restatable}{theorem}{infstbUpper}
    \label{infstb upper}
    $\ex^\icf_\infstb$ and $\cred^\icf_\infstb$ are each $\Pi^0_2$. $\skep^\icf_\infstb$ is $\Sigma^0_2$.
\end{restatable}

\begin{proof}
Let $\F$ be a computably finitary AF with $A_\F=\{a_n: n\in\omega\}$.  Consider the following condition, which says that $\F$ has finitely many stable extensions and no other stable extension:
   \begin{equation}  \tag{$\mathcal{E}$}
   (\exists n)(\exists e_0,\ldots,e_n)  
    (\{\hat{D}_{e_j} : j\leq n\}=\stb(\F)).
   \end{equation}

   
We argue that $\mathcal{E}$ is a $\Sigma^0_2$ condition. First, note that checking whether a given finite set $X$ is stable is a $\Pi^0_1$ condition. Next, given a list of finite stable extensions $X_1,\ldots X_n$, then we can define the tree of all stable extensions which do not contain any of the $X_i$: 
$
\bigcap_i \bigcup_{a\in X_i} \T^\F_{\stb+\emptyset-\{a\}}
$.
This is a computably finitely branching tree, and there are no other stable extensions iff the tree has no path, which is a $\Sigma^0_1$ condition by K\"onig's Lemma. Thus, the overall complexity of $\mathcal{E}$ is a $\Sigma^0_2$. 

Finally, we show that $\mathcal{E}$ is exactly the complement of $\ex^\icf_\infstb$. Obviously, if $\F$ has an infinite stable extension $S$ then $\mathcal{E}$  does not hold, since any finite list of finite stable extensions cannot include $S$. Conversely, if $\mathcal{E}$ does not hold, then there is either an infinite stable extension or there are infinitely many distinct finite stable extensions, but the latter implies that there is an infinite stable extension by Theorem \ref{Infinitely many finite implies one infinite - stable} with $D=E=\emptyset$. Hence, $\F$ has an infinite stable extension iff $\mathcal{E}$ fails; since $E$ is $\Sigma^0_2$, this proves that $\ex^\icf_\infstb$ is $\Pi^0_2$. 

The complexity of $\cred^\icf_\infstb$ and $\skep^\icf_\infstb$ can now be readily determined by reasoning as for the cases of $\cred^\icf_\infad$ and $\skep^\icf_\infad$ in  the proof of Theorem \ref{infad upper}. Indeed, the above argument still holds if we replace the notion of stable extension by ``stable extension containing $a$'': it is enough to use Theorem \ref{Infinitely many finite implies one infinite - stable} with $D=\{a\}$, $E=\emptyset$. This proves that  $\cred^\icf_\infstb$ is $\Pi^0_2$. Finally, by applying Theorem \ref{Infinitely many finite implies one infinite - stable} with $D=\emptyset$ $E=\{a\}$, we obtain that the argument also holds for the notion  of ``stable extension omitting $a$''; this proves that $\skep^\icf_\infstb$ is $\Sigma^0_2$.
\end{proof}

\begin{restatable}{theorem}{infstbLower}
    \label{infstb lower}
    $\ex^\icf_\infstb$ and $\cred^\icf_\infstb$ are $\Pi^0_2$-hard. $\skep^\icf_\infstb$ is $\Sigma^0_2$-hard.
\end{restatable}

\begin{restatable}{theorem}{uniinfstb}
    \label{uniinfstb}
    $\uni^\icf_\infstb$ is $\Pi^0_3$-complete.
\end{restatable}

\section{Conclusions}\label{sec:conclusions}

In this paper, we tackled the question of how the assumption of being finitary leads to a decrease in complexity of computational problems on AFs. We observed that this is not always the case, as in the semantics of $\cf,\na,\infcf,$ and $\infna$, where there was no decrease in complexity. On the other hand, for the admissibility-based semantics we considered, we found a significant drop of complexity to low levels of the arithmetical hierarchy. In fact, these complexities are so reduced that there are algorithms that approximate, via a limiting procedure, which arguments are credulously or skeptically accepted for $\ad,\co,\stb,\infad,\infstb$ (see Remark \ref{computable approximations}). The problems of credulous acceptance of arguments for admissible, stable, or complete semantics are co-c.e., meaning that all arguments can be accepted unless we witness a contradiction to that argument at some finite time. Skeptical acceptance of arguments for admissible, stable, or complete are c.e.\ problems, meaning that if an argument should be skeptically accepted, then we know this at finite time. 

For reasoning with each of these semantics, we conclude that the infinite but finitary argumentation frameworks provide a natural setting for reasoning which balances well the competing goals of being expressive enough to be applied to a myriad of reasoning settings, yet being computationally tractable enough for the analysis within the framework to be useful.

We find the drop to arithmetical levels of the decision problems for infinite semantics to be quite surprising. These rely on the combinatorial arguments in Theorems \ref{Infinitely many finie implies one infinite} and \ref{Infinitely many finite implies one infinite - stable}, which provide a logical compactness-style result for reasoning with admissible and stable semantics. 
We leave open the intriguing question of the complexity of $\skep^\icf_\infco$ and $\uni^\icf_\infco$.


\pagebreak

\section{Supplementary Material}\label{supplement}

\subsection{Details for Section \ref{sec: cfna}}

\subsubsection{Encoding extensions into trees}\label{appendix: encoding extensions into trees}

Theorem \ref{encoding extensions into trees} is a major technical tool for understanding the complexity of the classical semantics in both the non-finitary and finitary settings. The theorem as stated here is a generalization of several theorems from \cite{andrews2024FCR}. Essentially the same proofs work with some alterations: We do not need a branching at level $0$ to ensure that the extension is non-empty, since we accept the empty extension. We put $D$ into the set $\ins_\sigma$ for any $\sigma$ and $E$ into the set $\out_\sigma$ for any $\sigma$. Since $D$ and $E$ may be infinite, though in our application in this paper one will be empty and the other a single argument, the definition of $\sigma$ being on $\T$ only checks the properties given in \cite{andrews2024FCR} for numbers $< |\sigma|$. 

\thetreetheorem*

We present the proofs for this theorem separately in each of the cases of admissible, stable, and complete semantics. Theorem  \ref{encoding extensions into trees} is the combination  of Theorems \ref{admissible computable bijection ne}, \ref{stable computable bijection ne}, and \ref{encoding complete} proved below.

\subsubsection*{The admissible case}
Given an argumentation framework $\F=(A_\F,R_\F)$ with $A_\F=\{a_i : i\in \omega\}$ and sets $D,E\subseteq A_\F$, we will describe a tree $\T^\F$ so that the paths of $\T^\F$ encode the admissible extensions in $\F$ containing $D$ and disjoint from $E$.
We begin with an intuitive description of how a path $\pi$ through the tree $\T^\F$ will encode an admissible extension $S$, and we give the formal definition of $\T^\F$ below.

Branching in $\T^\F$ will come in two flavors. For any $j>0$, the branching on level $2j$ serves to code whether or not $j\in S$. Branching on the odd levels serve to explain how $S$ satisfies the hypothesis of being an admissible extension. If $a_i\att a_j$ is the $n$th element of some (computable, if $\F$ is computable) enumeration of all attacking pairs of arguments, then $\sigma(2n+1)$ will be $0$ if $a_j\notin S$ and otherwise will be $k+1$ where $k$ is least so that $a_k\in S$ and $a_k\att a_i$.


Let $(g_n)_{n\in\omega}$ be a (computable, if $\F$ is computable) sequence of all elements of $R_\F$. If $g_n=(a_i, a_j)$, we denote $a_i$ by $g_n^-$ and $a_j$ by $g_n^+$. We now formally define the tree $\T^\F$. 

\begin{definition}
Any given  string $\sigma\in\omega^{<\omega}$ defines two subsets of arguments in $A_\F$:
\begin{itemize}
\item $\ins_\sigma=D\cup \{a_j : \sigma(2j)=1\} \cup \{a_k :(\exists j)(\sigma(2j+1)=k+1)\}\cup \{a_i : (\exists j)(\sigma(2j+1)>0 \wedge g_j^+ = a_i)\}$;

\smallskip

\item $\out_\sigma=E\cup \{a_j:\sigma(2j)=0\} \cup  \{a_i : (\exists j)(\sigma(2j+1)>i+1 \wedge a_i\att g^-_j)\}\cup \{a_i : (\exists j)(\sigma(2j+1)=0\wedge g_j^+=a_i)\} $.
\end{itemize}

We define $\T^\F$ as the set of $\sigma$ so that 
\begin{itemize}
    \item $\ins_\sigma$ does not contain elements $a_i\att a_j$ with $i,j<|\sigma|$;
    \item $\ins_\sigma\cap \out_\sigma\cap \{a_i : i<|\sigma|\} = \emptyset$
    \item If $2j<|\sigma|$, then $\sigma(2j)\in \{0,1\}$;
    \item If $2j+1<|\sigma|$ and $\sigma(2j+1)=k+1$, then $a_{k}\att g_j^-$.
\end{itemize}

\end{definition}

Note that if $\F$ is computable, then $\T^\F$ is computable, and if $\F$ is computably finitary, then $\T^\F$ is computably finitely branching.

\begin{theorem}\label{admissible computable bijection ne}
    Let $\F$ be a (computable) argumentation framework. Then the admissible extensions of $\F$ which contain $D$ and are disjoint from $E$ are in (computable) bijection with the paths in $\T^\F$.
\end{theorem}
\begin{proof}
    Given a non-empty admissible extension $S$ of $\F$ which contains $D$ and is disjoint from $E$, we define the corresponding path $\pi$ in $\T^\F$ as follows. Let $\pi(2j)=1$ if $a_j\in S$ and $\pi(2j)=0$ otherwise. Let $\pi(2n+1)$ be $0$ if $g_n^+\notin S$ and be $k+1$ where $k$ is least so that $a_k\in S$ and $a_k\att g_n^-$ otherwise. It is straightforward to check that $\pi\in [\T^\F]$.

    Given a path $\pi$ through $\T^\F$, first note that whenever there is some $\sigma \prec \pi$ so that $a_n\in \ins_\sigma$, then $\pi(2n)=1$. This is because whenever $\sigma\preceq \tau$, then $\ins_\sigma\subseteq \ins_\tau$. Then since $\tau:=\pi\restriction_{\max(|\sigma|,2n+1)}$ is in $\T^\F$, we cannot have $\tau(2n)=0$ since $a_n\in \ins_\tau$. Thus $\pi(2n)=\tau(2n)=1$. It follows that $\bigcup_{\sigma\prec \pi} \ins_\sigma = \{a_n: \pi(2n)=1\}$. The same argument shows that $\bigcup_{\sigma\prec \pi} \out_\sigma = \{a_n: \pi(2n)=0\}$.
    Let $S=\{a_n: \pi(2n)=1\}$. This also shows that $D\subseteq S$ and $E\cap S = \emptyset$.
    
    Note that $S$ is conflict-free, since if $a_i,a_j\in S$ then there is some long enough $\sigma\prec \pi$ so that $a_i,a_j\in \ins_\sigma$ and $i,j<|\sigma|$. Since $\sigma\in \T^\F$, it follows that $a_i\natt a_j$. Next, observe that $S$ defends itself. If $a_i\att a_j$ and $a_j\in S$, then there is some $n$ so that $g_n= (a_i, a_j)$. Then consider $\sigma = \pi\restriction_{2n+2}$. We must have $\sigma(2n+1)=k+1$ for some $k$ with $a_k\in S$ and $a_k\att a_i$.

    Finally, note that both the map from $\pi$ to $S$ and from $S$ to $\pi$ are computable if $\F$ is computable and are inverses of each other.
\end{proof}

\subsubsection*{The stable case} 


Similarly, we can construct a tree encoding the stable extensions including $D$ and disjoint from $E$ by making $\sigma(n) = 0$ if $a_n\in S$ and otherwise making $\sigma(n)$ be $k+1$ where $k$ is least so that $a_k\in S$ and $a_k\att a_n$.

\begin{definition}
Any given  string $\sigma\in\omega^{<\omega}$ defines two subsets of arguments in $A_\F$:
\begin{itemize}
\item $\ins_\sigma=D\cup \{a_i: \sigma(i)=0\}\cup \{a_{\sigma(i)-1}: i<|\sigma|\wedge \sigma(i)>0\}$;

\smallskip

\item $\out_\sigma=E\cup \{a_i : i<|\sigma|\wedge \sigma(i)>0\}\cup 
\{a_i : (\exists j)\sigma(j)>i+1 \wedge a_i\att a_j)\}$.
\end{itemize}

We define $\T^\F$ as the set of $\sigma$ so that 
\begin{itemize}
    \item $\ins_\sigma$ does not contain elements $a_i\att a_j$ with $i,j<|\sigma|$;
    \item $\ins_\sigma\cap \out_\sigma\cap \{a_i : i<|\sigma|\} = \emptyset$
    \item If $j<|\sigma|$ and $\sigma(j)=k+1$, then $a_{k}\att a_j$.
\end{itemize}

\end{definition}

Note that if $\F$ is computable, then $\T^\F$ is computable, and if $\F$ is computably finitary, then $\T^\F$ is computably finitely branching.

\begin{theorem}\label{stable computable bijection ne}
    Let $\F$ be a (computable) argumentation framework. Then the stable extensions of $\F$ which contain $D$ and are disjoint from $E$ are in (computable) bijection with the paths in $\T^\F$.
\end{theorem}
\begin{proof}
    Given a stable extension $S$ of $\F$ which contains $D$ and is disjoint from $E$, we define the corresponding path $\pi$ in $\T^\F$ as follows. For each $n$, let $\pi(n)$ be 0 if $a_n\in S$ and let $\pi(n)$ be $k+1$ where $k$ is least so that $a_k\in S$ and $a_k\att a_n$ otherwise. It is straightforward to check that $\pi\in [\T^\F]$.

    Given a path $\pi$ through $\T^\F$, first note that whenever there is some $\sigma\prec \pi$ so that $a_n\in \ins_\sigma$, then $\pi(n)=0$. This is because whenever $\sigma\preceq \tau$, then $\ins_\sigma\subseteq \ins_\tau$. Then since $\tau=\pi\restriction_{\max(|\sigma|,n+1)}$ is on $\T^\F$, we cannot have $\tau(n)\neq 0$ since $a_n\in \ins_\tau$ and therefore cannot be in $\out_\tau$. It follows that $\bigcup_{\sigma\prec \pi} \ins_\sigma = \{a_n: \pi(n)=0\}$. Let $S=\{a_n: \pi(n)=0\}$. This also shows $D\subseteq S$ and $E\cap S = \emptyset$.

    Note that $S$ is conflict-free, since if $a_i,a_j\in S$, then $a_i\not\att a_j$ since $\sigma = \pi\restriction_{\max(i,j)+1}$ is in $\T^\F$, and thus $\ins_\sigma$ is conflict-free. Next observe that for any $n$, either $a_n\in S$ or there is some $m$ so that $a_m\in S$ and $a_m\att a_n$. In particular, if $\pi(n)=0$, then $a_n\in S$ and otherwise $a_{\pi(n)-1}\in S$ and $a_{\pi(n)-1}\att a_n$. 

    Finally, note that both the map from $\pi$ to $S$ and from $S$ to $\pi$ are computable if $\F$ is computable and are inverses of each other.
\end{proof}

\subsubsection*{The complete case}

Given an argumentation framework $\F$ and $D,E\subseteq A_\F$, we can similarly construct a tree $\T^\F$ so that paths through $\T^\F$ code complete extensions containing $D$ and disjoint from $E$. In order to ensure that $f_\F(S)\subseteq S$, we will need the paths in $\T^\F$ to not only code sets $S$ but also their attacked sets $S^+$. 

Given an extension $S$, we will let $\pi\in \T^\F$ encode $S$ and $S^+$ as follows:
\begin{itemize}
    \item $\pi(2n) = 0$ if $a_n\in S$ and otherwise $\pi(2n)=k+1$ where $k$ is least so $a_k\notin S^+$ and $a_k\att a_n$.
    \item $\pi(2n+1) = 0$ if $a_n\notin S^+$ and otherwise $\pi(2n+1)=k+1$ where $k$ is least so $a_k\in S$ and $a_k\att a_n$.
\end{itemize}

Note that $\pi(2n)$ explains why $a_n$ is either in $S$ or it is not in $f_\F(S)$, i.e., $f_\F(S)\subseteq S$, while $\pi(2n+1)$ simply verifies that the elements which $\pi$ says are in $S^+$ are in fact in $S^+$.

Formally, we define $\T^\F$ as follows:

\begin{definition}\label{Tree of complete extensions}
Any given  string $\sigma\in\omega^{<\omega}$ defines four subsets of arguments in $A_\F$:
\begin{itemize}
\item $\ins_\sigma= D\cup 
\{a_i : \sigma(2i)=0\} 
\cup 
\{a_k :(\exists j)(\sigma(2j+1)=k+1\}$ 

\smallskip

\item $\out_\sigma= E\cup 
\{a_i: \sigma(2i)\neq 0\} 
\cup 
\{a_i : (\exists j)\sigma(2j+1)>i+1 \wedge a_i\att a_j)\} 
$
\smallskip

\item $\text{InSplus}_\sigma = 
\{a_i : (\exists j)(\sigma(2j)>i+1\wedge a_i\att a_j )\} 
\cup
\{a_i: \sigma(2i+1)\neq 0\} 
$
\item $\text{OutSplus}_\sigma = 
\{a_i : (\exists j) \sigma(2j)=i+1\} 
\cup
\{a_i : \sigma(2i+1)=0\} 
$
\end{itemize}

We define $\T^\F$ as the set of $\sigma$ so that 
\begin{enumerate}
    \item $\ins_\sigma$ does not contain elements $a_i\att a_j$ with $i,j<|\sigma|$;
    \item $\ins_\sigma\cap \out_\sigma\cap \{a_i : i<|\sigma|\} = \emptyset$;
    \item $\text{InSplus}_\sigma\cap \text{OutSplus}_\sigma = \emptyset$;
    \item If $\sigma(2j)=k+1$, then $a_{k}\att a_j$;
    \item If $\sigma(2j+1)=k+1$, then $a_k \att a_j$;
    \item For $j,k<|\sigma|$, if $a_k\in \text{OutSplus}_\sigma$ and $a_j\in \ins_\sigma$ then $a_j\not\att a_k$;
    \item For $n,m<|\sigma|$, if $a_n\in \text{OutSplus}_\sigma$ and $a_m\in \ins_\sigma$ then $a_n\not\att a_m$ (i.e., $\sigma$ does not contradict $S\subseteq f_\F(S)$).
\end{enumerate}

\end{definition}

Note that if $\F$ is computable, then $\T^\F$ is computable, and if $\F$ is computably finitary, then $\T^\F$ is computably finitely branching.

\begin{theorem}\label{encoding complete}
    The complete extensions of $\F$ which contain $D$ and are disjoint from $E$ are in bijection with the set of paths $[\T^\F]$. 
\end{theorem}
\begin{proof}
    Let $S$ be any complete extension containing $D$ and disjoint from $E$. We can define $\pi$ from $S$ as follows: 
    \begin{itemize}
        \item $\pi(2n)$ be 0 if $a_n\in S$ and otherwise $\pi(2n)=k+1$ where $k$ is least so $a_k\notin S^+$ and $a_k\att a_n$.
        \item $\pi(2n+1)=0$ if $a_n\notin S^+$ and otherwise $\pi(2n+1)=k+1$ where $k$ is least so $a_k\in S$ and $a_k\att a_n$.
    \end{itemize} 
    It is straightforward to verify that each condition (1-7) of Definition \ref{Tree of complete extensions} is satisfied by $\pi\restriction_n$ for each $n$.

    Given a path $\pi\in [\T^\F]$, we can define sets $S=\{i : \pi(2i)=0\}$ and $U=\{i : \pi(2i+1)\neq 0\}$. We note that when $\sigma \preceq \tau$, $\ins_\sigma\subseteq \ins_\tau$. It follows from this fact, as in the previous theorems, that $S=\bigcup_{\sigma \prec \pi} \ins_\sigma$. Thus $S$ contains $D$ and is disjoint from $E$. Similarly, $U=\bigcup_{\sigma\prec \pi} \text{InSplus}_\sigma$. 
    
    Next we see that $U=S^+$. If $n\in U$, then $\pi(2n+1)\neq 0$ and by condition 5, we have $a_{\pi(2n+1)-1}$ attacks $a_n$. But then $a_{\pi(2n+1)-1}\in \ins_{\pi\restriction_{2n+2}}$, so $a_{\pi(2n+1)-1}\in S$. Thus $U\subseteq S^+$. On the other hand if $n\notin U$, then condition 6 for all $\sigma$ of length $>2n+1$ ensures that there is no $a_j\in S$ so $a_j\att a_n$. Thus $S^+\subseteq U$.

    Next we verify that $S$ is complete. $S$ is clearly conflict free by Condition 1. Condition 7 ensures that any $a_m\in S$ is also in $f_\F(S)$, since if $n\notin U$, then $a_n\not\att a_m$. To see that $f_\F(S)\subseteq S$, note that any $a_n\notin S$ has $\pi(n)\neq 0$ and $a_{\pi(n)-1}\notin S^+$ and $a_{\pi(n)-1}\att a_n$ by condition 4. Thus $a_n\notin f_{\F}(S)$.

    Finally, note that the map from $S$ to $\pi$ and $\pi$ to $S$ are inverses of each other.
\end{proof}

Note that in the complete case, we do not have a computable bijection between the paths in $\T^\F$ and the complete extensions in $\F$. Rather, we get a computable bijection between the paths in $\T^\F$ and the pairs $(S,S^+)$ where $S$ is a complete extension in $\F$. As such, we have a computable map from the paths in $\T^\F$ to the complete extensions, but since $S^+$ is not uniformly computed from $S$, we have no computable way to take a complete extension and find the corresponding path in $\T^\F$.

%
%
%




\subsubsection{Computational problems for conflict-free and naive semantics}\label{supplement: cf and na}




We write $\varphi(n)[s]$ to denote the
outcome of computing $\varphi$, on input $n$, for $s$ many steps: if such
computation converges, namely $\varphi(n)$ halts within $s$ many steps, we write $\varphi(n)[s]\downarrow$, otherwise we write $\varphi(n)[s]\uparrow$.

\allcf*


\begin{proof}
$(i)$ An argument $x$ belongs to some conflict-free extension of a given AF $\F$ iff 
$x$ is not self-defeating, i.e, $x\natt x$. Indeed, if $x\natt x$ holds, then $\{x\}\in \cf({\F})$; on the other hand, $x\att x$ implies that no conflict-free extension can contain $x$. Since the attack relation in $\F_e$ is uniformly computable from the index $e$ (recall from Remark \ref{rem:inTot} that we discuss complexity inside \tot), we deduce that $\cred_\cf^\icf$ is computable.

$(ii)$ In the last item, we observed that a given AF $\F$ has a non-empty conflict-free extension iff $\F$ contains some argument $x$ so that $x\not\att x$. As the attack relation is computable and the existential quantifier ranges over arguments, this immediately proves that the set $\nemp^\icf_\cf$ is $\Sigma^0_1$. To see that $\nemp^\icf_\cf$ is $\Sigma^0_1$-hard (and therefore $\Sigma^0_1$-complete), we dynamically define a  sequence of computable AFs $(\F^K_n)_{n\in\omega}$ as follows: At each stage $s$, we impose that the argument $s$ is self-defeating in $\F_n$ iff we witness that $\phi_n(n)[s]\uparrow$. The effect of this action is that of reducing the halting set $K$ into $\nemp_\cf^\icf$: indeed, if $x\in K$, then $\F^K_x$ has non-empty extensions as it contains  non-self-defeating arguments; otherwise, $x\att x$ for all arguments in $\F$, thus the only conflict-free extension in $\F_x$ would be the empty one. Since $K$ is $\Sigma^0_1$-hard, we conclude that $\nemp^\icf_\cf$ is $\Sigma^0_1$-complete.

$(iii)$ Since the empty set is a conflict-free extension, we have that a given  $\F$ has unique conflict-free extension iff $\F$ has no non-empty conflict-free extension. Hence, $\uni^\icf_\cf$ is the complement of $\nemp^\icf_\cf$ (in $\tot$), which shows that  $\uni^\icf_\cf$ is $\Pi^0_1$-complete.
\end{proof}

\allnaive*


\begin{proof}
    $(i)$ and $(iii)$: It follows from a standard Zorn's Lemma argument that every conflict-free extension is contained in a naive extension. Hence, $\cred^\icf_\na$ equals $\cred^\icf_\cf$, and therefore is computable. For the same reason,  $\nemp^\icf_\na$ equals $\nemp^\icf_\cf$, and therefore is $\Sigma^0_1$-complete. 

    $(ii)$: To prove that $\skep_\na^\icf$ is $\Pi^0_1$, we will prove that, for all AF $\F$, $a\in\skep_\na^\icf(\F)$ iff $a$ is non-self-defeating and the following condition $(\dagger)$ holds: $(\forall y \in A_\F)(y\not\att y \Rightarrow \{y,a\}\in \cf(\F))$.  Indeed, suppose that $a$ is in every naive extension and let $y\in A_\F$ be non-self-defeating. By Zorn's lemma, there must be a naive extension $U$ containing $y$ (since $\{y\}$ is conflict-free), but then $\{y,a\}\subseteq U$, since $a$ is  in every naive extension. It follows that $\{y,a\}$ is conflict-free, as desired. Conversely, if $a$ has no conflict with any element which is non-self-defeating, then whenever $U$ is conflict-free, then $U\cup \{a\}$ is also conflict-free, so $a$ is contained in every naive extension. Thus, $\skep_\na^\icf$ is $\Pi^0_1$, since this is the complexity of condition $(\dagger)$. The completeness is obtained by defining a computable $\F$ with $A_\F:=\{a_n : n\in\omega\} \cup \{b_k : k\in\omega\}$ and the attack relations: $a_n\att b_{\langle n, s\rangle}$ if $s$ is the least stage so that $\phi_n(n)[s]\downarrow$. 
In light of condition $(\dagger)$, the map $n\mapsto a_n$ gives a reduction from $\overline{K}$ to $\skep^\icf_\na(\F)$, thus proving that $\skep^\icf_\na$ is $\Pi^0_1$-complete.

$(iv)$: We observe that $\F$ has a unique naive extension iff the set $S=\{x: x\not\att x \}$ of the non-self-defeating arguments in $\F$ is conflict-free. Indeed, if $\F$ contains arguments $x\not\att x$ and $y\not\att y$ so that $\{x,y\}$ is not conflict-free, then $\{x\}$ and $\{y\}$ are contained in distinct naive extensions. Thus, $\F$ having a unique naive extension implies that $S$ is conflict-free. On the other hand, if $S$ is conflict-free, then every conflict-free extension is contained in $S$, so $S$ is the unique naive extension in $\F$.
This shows that $\uni_\na^\icf$ is $\Pi^0_1$, since $\F$ has unique naive extension iff $(\forall x,y)(x\not\att x \, \& \, y\not\att y \Rightarrow \{x,y\}\in\cf(\F))$. For proving that $\uni_\na^\icf$ is $\Pi^0_1$-complete, we define a sequence of computable AFs $(\F_n)_{n\in\omega}$ so that  $\F_n=\{a_k : k\in \omega\}\cup \{b_k : k\in \omega\}$ and $a_k\att b_k$ iff $k$ is least so that $\phi_n(n)[k]\downarrow$. 
This encoding ensures that $\F_n$ has a unique naive extension iff $n\notin K$, which proves that $\uni_\na^\icf$ is $\Pi^0_1$-complete.
\end{proof}

\subsubsection{Computational problems for admissible, stable, and complete semantics}\label{supplement: ad-based}

A few of the following proofs use the following theorem along with Theorem \ref{encoding extensions into trees}.

\begin{theorem}[\cite{andrews2024FCR}, Lemma 4.1]\label{encoding trees into AFs}
    Let $\T\subseteq \omega^{<\omega}$ be a tree. Then, there is an AF $\F^\T$	so that each non-empty extension $S$ of $\F^\T$ is stable iff $S$ is complete iff $S$ is admissible. Further, the set of stable extensions in $\F^\T$ are in computable bijection with the set of paths in $\T$. 
    
    Further, if $\T$ is computably finitely branching,  then $\F^\T$ is computably finitary.

    Further, there is one element $a_\lambda$ which is contained in any stable extension of $\F^\T$.
\end{theorem}

\easyuppers*

 
 \begin{proof}
Let $\F$ be a computably finitary AF. By Theorem  \ref{encoding extensions into trees}, $\F$  has a $\sigma$ extension iff ${\T}^{\F}_{\sigma + \emptyset - \emptyset}$
has a path. Since ${\T}^{\F}_{\sigma + \emptyset - \emptyset}$ is computably finitely branching (as $\F$ is assumed computably finitary), K\"onig's Lemma
shows that determining that the tree $\T^\F_{\sigma+\emptyset-\emptyset}$ has a path is a $\Pi^0_1$ problem. This shows that $\ex^\icf_\stb$ is $\Pi^0_1$. Similarly, since a given argument $x$ belongs to a $\sigma$-extension of $\F$ iff ${\T}^{\F}_{\sigma + \{x\} - \emptyset}$ has a path, $\cred^\icf_\sigma$ is $\Pi^0_1$.


The hardness results follow from Theorem \ref{encoding trees into AFs}: Indeed, given a computably finitely branching tree $\T$, consider the AF $\F^{\T}$. Then the argument $a_\lambda \in A_{\F^\T}$ is in $\cred^\icf_\sigma(\F)$ iff the tree $\T$ has a path. It is a standard easy fact of computability theory that the set of indices for computably finitely branching trees which have a path is a $\Pi^0_1$-hard problem. Similarly, $\F^\T\in \ex^\icf_\stb$ iff the tree $\T$ has a path, thus is also $\Pi^0_1$-complete.
\end{proof}

\OTHupper*


\begin{proof}
$(i)$ Since the empty set is never a stable extension,  we immediately see that  $\nemp^\icf_\stb = \ex^\icf_\stb$, which is $\Pi^0_1$-complete by Theorem \ref{easy uppers}. 

$(ii)$ Let $\sigma\in \{\ad,\co\}$. We have that an AF $\F$ has a non-empty $\sigma$-extension iff $(\exists a \in A_\F)(a \in \cred^\icf_\sigma(\F))$, which is $\Sigma^0_2$ since $\cred^\icf_\sigma$ is $\Pi^0_1$ by Theorem \ref{easy uppers}. To prove hardness, we reduce
 any given set $X$ in $\Sigma^0_2$ to $\nemp^\icf_\sigma$. 
 Since the set of indices of computably finitely branching trees which have a path is a $\Pi^0_1$-hard problem, for each $n$, there exists a sequence of computably finitely branching trees $(\T_i)_{i\in\omega}$ such that $n\in X$ iff there is some $i$ so that $\T_i$ has a path. We let $\F_n=\bigcup_i \F^{\T_i}$. Then the admissible or complete extensions of $\F_n$ are the unions of the admissible or complete extensions of each $\F^{\T_i}$. But Theorem \ref{encoding trees into AFs} shows that these extensions are each either empty or represent a path through $T_i$. Thus $\F_n$ has a non-empty admissible or complete extension iff there is some $i$ so that $\T_i$ has a path iff $n\in X$. 

$(iii)$ 
Note that for $\sigma\in \{\co,\stb\}$, $a\in \skep^\icf_\sigma(\F)$ iff $\T^\F_{\sigma+\emptyset-\{a\}}$ has no path, which is $\Sigma^0_1$ by K\"onig's Lemma.

 Let $W$ be a $\Sigma^0_1$ set. We prove that there is a computably finitary AF $\F$ with the grounded extension $G$ (recall the grounded extension is the unique smallest complete extension) equaling $\skep^\icf_\stb(\F)$ and encoding $W$.

    We fix a computable enumeration of $W$ which enumerates nothing at stage $0$ and enumerates at most one element at every later stage.
    Let $A_\F = \{a_{n,0} : n\in \omega\} \cup \{a_{n,2j-1},a_{n,2j} : j>0, n\notin W[j]\}$, and let $a_{n,m}$ attack $a_{n',m'}$ iff $n=n'$ and $m=m'+1$.

    Note that $a_{n,0}$ is the end of an infinite chain of attacks if $n\notin W$. If $n\in W$, then the greatest $m$ so that $a_{n,m}$ exists is even. The grounded extension contains exactly $G=\{a_{n,2k}\in A_\F : n\in W\}$. Since $n\in W$ iff $a_{n,0}\in G$, we see that $\skep^\icf_\co(\F)$ encodes $W$, showing that $\skep^\icf_\co$ is $\Sigma^0_1$-complete.

    We note that since every stable extension is complete, $G\subseteq S$ for any stable extension $S$. Note that both $G\cup \{a_{n,2k}\in A_\F : n\notin W\}$ and $G\cup \{a_{n,2k+1}\in A_\F : n\notin W\}$ are stable, so $\skep_\stb(\F)=G$.

$(iv)$
Note that $\F\in \uni_\sigma^\icf$ iff $\F\in \ex^\icf_\sigma$ and $\forall x (x\in \cred_\sigma^\icf \Rightarrow x\in \skep^\icf_\sigma)$, which is $\Pi^0_2$ due to Theorems \ref{easy uppers} and $(iii)$. 
Since the empty set is always admissible, we have that $\uni^\icf_\ad$ is the complement of $\nemp^\icf_\ad$, and thus is $\Pi^0_2$-hard.
In the construction of $\F_n$ in Theorem \ref{unistb upper}, we note that the empty set is a complete extension, so $\F_n$ has a unique complete extension iff it does not have a non-empty complete extension, thus $\uni^\icf_\co$ is $\Pi^0_2$-hard.

For each computable tree $\T$, we can define a tree $\hat{\T}$ by $\sigma\in \hat{\T}$ if $\sigma = 0^n$ for some $n$ or $\sigma = 1^\smallfrown \tau$ and $\tau\in \T$. Note that $\hat{\T}$ has exactly one more path than $\T$, namely the infinite string with all 0's.
To see that $\uni^\icf_\stb$ is $\Pi^0_2$-hard, we apply the same construction where we use the trees $\hat{T}_i$ in place of $T_i$. Then $\F_n$ has a stable extension, corresponding to choosing the 0-path in each $\hat{T}_i$. It has a unique stable extension unless one of the trees $T_i$ has a path, showing the $\Pi^0_2$-hardness of $\uni_\stb$.
\end{proof}

\subsection{Details for Section \ref{sec:infcf infna}}\label{supplement: sect 6}

\credexinfcfinfna*
\begin{proof}
 $(i)$:  Let $\T\subseteq \omega^{<\omega}$ be a given computable tree. We produce an AF so that the conflict-free extensions correspond exactly to the subsets of paths through $\T$. Fix a computable bijection $f:\omega\rightarrow \T$. Let $\F$ be so that $A_\F=\{a_i : i\in \omega\}$ and $a_i\att a_j$ iff $i<j$ and $f(i)$ is incomparable with $f(j)$. Note that  $\F$ is finitary, since we require $i<j$ for $a_i$ to attack $a_j$. Further, for each $a_j$, we can uniformly compute the set of attackers of $a_j$, thus $\F$ is computably finitary. It is straightforward to check that a conflict-free extension is (via $f$) in computable bijection with an infinite subset of a path in $\T$. Thus, $\F$ has an infinite conflict-free extension iff $\F$ has an infinite conflict-free extension containing $a_0$ iff $\T$ has a path. Thus, we have reduced a $\Sigma^1_1$-complete set, namely $\mathbf{Path}$, to $\ex^\icf_\infcf$ and $\cred^\icf_\infcf$, showing that they are $\Sigma^1_1$-hard. They are also, by definition, $\Sigma^1_1$.

$(ii)$ By the fact that every conflict-free extension is contained in a naive extension, it is immediate to observe that $\ex^\icf_\infna=\ex^\icf_\infna$ and  $\cred^\icf_\infna= \cred^\icf_\infcf$. By item $(i)$ of this proof, we deduce that  $\cred^\icf_\infna$ and $\ex^\icf_\infna$ are both $\Sigma^1_1$-complete. 

$(iii)$:   Given an AF $\F$, we form $\hat{\F}$ by adding a single argument $a$ so $a\att a$. By item $(i)$ of this proof, we know that $\ex^\icf_\infcf$ is $\Sigma^1_1$-complete. Since $a\in \skep^\icf_\infcf(\hat{F})$ iff $\F\notin\ex^\icf_\infcf$, $\skep^\icf_\infcf$ is $\Pi^1_1$-hard. It is also $\Pi^1_1$ by definition.
By the same reasoning, $\skep^\icf_\infna$ is $\Pi^1_1$-complete.
\end{proof}

For the remaining proofs in this section, we use the following benchmark complexity set:

$\mathbf{UniPath} = \{e : e$ is a computable index for a tree with exactly one path$\}$ is $\Pi^1_1$-complete.

One would expect $\mathbf{UniPath}$ to be d-$\Sigma^1_1$. It is a non-obvious Theorem from descriptive set theory (see \cite[Theorem 18.11]{kechris2012classical}) that $\mathbf{UniPath}$ is in fact $\Pi^1_1$.

\uniinfcf*
\begin{proof}
$(i)$    If $S$ is an infinite conflict-free extension and $a\in S$, then $S\smallsetminus \{a\}$ is also an infinite conflict-free extension. Hence, there is no AF with a unique infinite conflict-free extension.


$(ii)$   To see that $\uni^\icf_\infna$ is $\Pi^1_1$-hard, consider the encoding of a given computable tree $\T\subseteq \omega^{<\omega}$ into the conflict-free extensions of an AF $\F_{\T}$ defined in Theorem \ref{cred infcf infna}, item $(i)$. Observe that the naive extensions of $\F_{\T}$ are in computable bijection with the paths through $\T$. Hence, $\F_\T$ is in $\uni^\icf_\infna$ iff $\T$ has a unique path, which is the $\Pi^1_1$-complete problem $\mathbf{UniPath}$. 

It remains to be shown that $\uni^\icf_\infna$ is $\Pi^1_1$. On first sight, one would expect it to be d-$\Sigma^1_1$. To show that it is $\Pi^1_1$, we want to reduce it to $\mathbf{UniPath}$, thus we want to find a tree whose paths are in bijection with the naive extensions in $\F$.

\begin{restatable}{lemma}{infnalem}\label{lem:infnalem}
For any (computable) AF $\F$, there is a (computable) tree $\T$ so that the paths through $\T$ are in (computable) bijection with the infinite naive extensions in $\F$.  
\end{restatable}
\begin{proof}
    We use two types of branching in $\T$. On the $2j$th layer, we branch to identify the least $k>j$ so that $a_k$ is in the naive extension $S$. This ensures that $S$ is infinite. On the $2j+1$th layer, we branch to explain that either $a_j\in S$ (by extending by 0) or to give the least $n$ so that $a_n\in S$ and $\{a_j,a_n\}$ is not conflict-free (by extending by $n+1$). This ensures naiveness of $S$.

        \begin{definition}
        Any given  string $\sigma\in\omega^{<\omega}$ defines two subsets of arguments in $A_\F$:
        \begin{itemize}
        \item $\ins_\sigma=\{a_{\sigma(2j)}:2j<|\sigma|\}\cup \{a_{j}: \sigma(2j+1)=0\}\cup \{a_{\sigma(2j+1)-1} : \sigma(2j+1)>0\}$;
        
        \smallskip
        
        \item $\out_\sigma=\{a_i : (\exists j) j<i<\sigma(2j)\}\cup 
        \{a_j : \sigma(2j+1)>0\}\cup \{a_k : k<\sigma(2j+1)-1\wedge \{a_j,a_k\}\text{ not conflict-free}\}$.
        \end{itemize}
        
        We define $\T^\F$ as the set of $\sigma$ so that 
        \begin{itemize}
            \item $\ins_\sigma$ is conflict-free;
            \item $\ins_\sigma\cap \out_\sigma = \emptyset$
            \item For no $j$ is $\sigma(2j)\leq j$.
            \item If $\sigma(2j+1)=k+1$, then $\{a_{k},a_j\}$ is not conflict-free.
        \end{itemize}
        \end{definition}

    We now argue that the paths through $\T$ are in computable bijection with the infinite naive extensions in $\F$.

    If $S$ is an infinite naive extension in $\F$, then we can define $\pi$ by $\pi(2j)$ is the least $k>j$ so that $a_k\in S$ and $\pi(2j+1)$ is $0$ if $j\in S$ and otherwise is $k+1$ where $k$ is least so that $a_k\in S$ and $\{a_k,a_j\}$ is not conflict-free. This is well-defined since $S$ is naive, and thus $S\cup \{a_j\}$ cannot be a proper superset of $S$ and also conflict-free. It is straightforward to check that $\pi\in [\T]$.

    If $\pi\in [\T]$, we define $S=\{j : \sigma(2j+1)=0\}$. As in the previous cases, $S=\bigcup_{\sigma\preceq \pi}\ins_\sigma$. $S$ is conflict-free since if $a_i,a_j\in S$, then there is some $\sigma$ with $a_i,a_j\in \ins_\sigma$ and $\ins_\sigma$ is conflict-free. For any $a_k\notin S$, $S\cup \{a_k\}$ is not conflict-free because $S$ contains $a_{\pi(2k+1)}$ which is in conflict with $a_k$. Finally, the map from $S$ to $\pi$ and vice-versa are computable if $\F$ is computable, and are inverses of each other.
\end{proof}


It follows that $\F$ is in $\uni^\icf_\infna$ iff the tree $\T$ guaranteed by Lemma \ref{lem:infnalem} is in $\mathbf{UniPath}$, which is a $\Pi^1_1$-condition.
\end{proof}


\subsection{Details for Section \ref{sec:infad}}\label{supplement: sec 7}

\uniinfad*
\begin{proof}
Observe that a computably finitary AF $\F$ has unique infinite admissible semantics iff  $\F\in \ex^\icf_\infad$ and $(\forall a\in A_\F)(a\in \cred^\icf_\infad (\F)\Rightarrow a\in \skep^\icf_\infad(\F))$, which, in the light of Theorem \ref{infad upper}, is a $\Pi^0_3$ condition.

    Next, we show $\Pi^0_3$-hardness. Let $X$ be a $\Pi^0_3$ set. Since $\Fin$ is $\Sigma^0_2$-complete, we can find a uniformly computable sequence of c.e.\ sets $(V^n_i)_{i,n\in\omega}$ so that $n\in X$ iff $V^n_i$ is finite for each $i$. We may also assume that the sets $V^n_i$ enumerate numbers in order (since we only care about cardinality, we simply re-order the set if necessary). Moreover, we may also assume that $0$ enters $V^n_i$ at stage $0$ for each $n,i$. 


    Let $\G_n=(A_{\G_n},R_{\G_n})$ be defined as follows (see Figure \ref{figure: uni pi3-hard}). Let 
\begin{multline*}
     A_{\G_n} =\{a_i : i\in \omega\} \cup \{c^i_{\langle j,m \rangle} : m \text{ is the $j$th} \text{ expansionary stage for $V^n_i$}\}.
\end{multline*}  
   
 We let $R_\G$ contain exactly the following attack relations:
    \begin{itemize}
        \item $a_i\att a_i$ if $i$ is odd;
        \item $c^i_{\langle j,m\rangle} \att c^i_{\langle j,m\rangle}$ if $j$ is odd;
        \item $a_i \att a_{i+1}$ for every $i$;
        \item $a_{2i} \att c^i_{\langle 0,0\rangle}$ and $c^i_{\langle 0, 0 \rangle}\att a_{2i}$;
        \item $a_{2i-1}\att c^i_{\langle 0,0\rangle}$;
        \item $c^i_{\langle j,m\rangle} \att c^i_{\langle j',m'\rangle}$ iff $j'=j+1$.
    \end{itemize}

    We refer to $c^i_{\langle j,m \rangle}$ as $d^i_j$ (the $m$ records the $j$th expansionary stage, but we only need to know that one exists).
    Observe that the set $\{a_{2i} : i\in \omega\}$ is always an admissible extension. Furthermore, it is straightforward  to check that if $V^n_i$ is infinite (and thus $n\notin X$), then $\{a_{2j} : j<i\}\cup \{d^i_{j} : j \text{ is even}\}$ is another admissible extension.

\begin{figure}\centering
\scalebox{0.75}{
\begin{tikzpicture}[->=Stealth]

    \node[circle, draw, 
    ] (a0) {$a_0$};
    \node[circle, draw, 
    ] (a1) [right=of a0] {$a_1$};
    \node[circle, draw, 
    ] (a2) [right=of a1] {$a_2$};
    \node[circle, draw, 
    ] (a3) [right=of a2] {$a_3$};
    \node[circle, draw, 
    ] (a4) [right=of a3] {$a_4$};
    \node[circle, draw, 
    ] (a5) [right=of a4] {$a_5$};
    \node[circle, draw, 
    ] (a6) [right=of a5] {$a_6$};
   \node[draw=none] (ellipsis2)[right=1mm of a6] {$\cdots$};

    \node[circle, draw, 
    ] (b1) [below=of a2] {$d^1_0$};
    \node[circle, draw, 
    ] (b2) [below=of a4] {$d^2_0$};
    \node[circle, draw, 
    ] (b3) [below=of a6] {$d^3_0$};

     \node[circle, draw, 
     ] (d11) 
    [below=of b1] {$d^1_1$};
    \node[circle, draw, 
    ] (d12) [below=of d11] {$d^1_2$};

         \node[circle, draw, 
         ] (d21) 
    [below=of b2] {$d^2_1$};

             \node[circle, draw, 
             ] (d31) 
    [below=of b3] {$d^3_1$};
    \node[circle, draw, 
    ] (d32) [below=of d31] {$d^3_2$};
   \node[draw=none] (ellipsis2)[below=1mm of d32] {$\vdots$};

\draw[->] (a0) edge (a1);
\draw[->] (a1) edge (a2);
\draw[->] (a2) edge (a3);
\draw[->] (a3) edge (a4);
\draw[->] (a4) edge (a5);
\draw[->] (a5) edge (a6);

\draw[->] (a1) edge (b1);
\draw[->] (a3) edge (b2);
\draw[->] (a5) edge (b3);

\draw[->] (b1) edge (d11);
\draw[->] (d11) edge (d12);

\draw[->] (b2) edge (d21);

\draw[->] (b3) edge (d31);
\draw[->] (d31) edge (d32);


\draw[<->] (a2) edge (b1);
\draw[<->] (a4) edge (b2);
\draw[<->] (a6) edge (b3);

\draw[->] (a1) edge[loop above, line width=0.65pt] (a1);
\draw[->] (a3) edge[loop above, line width=0.65pt] (a3);
\draw[->] (a5) edge[loop above, line width=0.65pt] (a5);

\draw[->] (d11) edge[loop left, line width=0.65pt] (d11);
\draw[->] (d21) edge[loop left, line width=0.65pt] (d21);
\draw[->] (d31) edge[loop left, line width=0.65pt] (d31);

\end{tikzpicture}
}

\caption{A fragment of $\G_n$ where $|V^n_1|=3$, $|V^n_2|=2$, and $|V^n_3|\geq 3$.}
\label{figure: uni pi3-hard}

\end{figure}
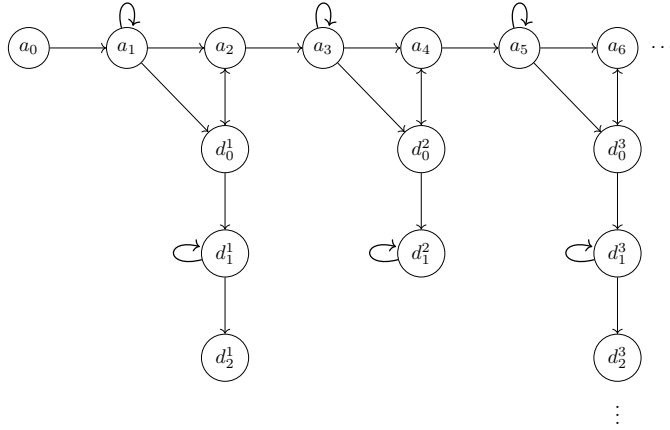

    \begin{lemma}
        If every $V^n_i$ is finite, then $\{a_{2i} : i\in \omega\}$ is the only infinite admissible extension.
    \end{lemma}
    \begin{proof}
        Suppose that $S$ is an infinite admissible extension.
        Observe that no element $a_i$ with $i$ odd or $d^i_{j}$ with $j$ odd can be contained in $S$, as all these  arguments are self-defeating. 

        Next, note that if $a_{2k}\in S$ with $k>0$, then $a_{2k-2}\in S$ since the attack from $a_{2k-1}$ must be defended. We now distinguish two cases:
\begin{itemize}
\item \emph{Case 1.} There is some argument $d^i_{j}$ contained in $S$. If so, then  $d^i_{j-2}\in S$ since the attack from $d^i_{j-1}$ must be defended. Iterating this reasoning, we see that $d^i_{0}$ is contained in $S$. This in turn implies that $a_{2i-2}\in S$ since the attack from $a_{2i-1}$ must be defended. Thus each $a_{2m}$ for $m<i$ is in $S$. Also, $a_{2i}\notin S$ since it is in conflict with $d^i_{0}$.  It follows that there cannot be any argument $d^k_{j}$ or $a_{2k}$ with $k>i$, since this would imply $a_{2i}\in S$. Thus, $S$ is contained in $\{a_{2m} : m<i\}\cup \{d^m_j : m< i \}$, so $S$ is finite, which is a contradiction.
\item \emph{Case 2.} There is no argument $d^i_{j}\in S$. Then $S$ is $\{a_{2k}: k\in J\}$ for $J$ an initial segment of $\omega$, but since $S$ is infinite, we see that $S$ is exactly $\{a_{2i} : i\in \omega\}$.
\end{itemize}        
Hence, if each $V^n_i$ is finite, then $|\infad(\F_n)|=1$.
    \end{proof}

It follows that $n\in X$ iff there is a unique infinite admissible extension in $\F_n$. Since $X$ was chosen as an arbitrary $\Pi^0_3$ set, we conclude that $\uni^\icf_\infad$ is $\Pi^0_3$-complete.
\end{proof}

\subsection{Details for Section \ref{sec:infstb}}\label{supplement: sec 8}

\combinatoricsStable*

\begin{proof}
    Observe that if $S$ is an infinite stable extension of $\F$ containing $D$ and disjoint from $E$, then $S\subseteq Y$, so $Y$ is infinite.

    Now suppose that $Y$ is infinite. For each $b\in Y$, fix a stable extension $X_b$ with $D\cup \{b\}\subseteq X_b$ and $E\cap X_b=\emptyset$. If any $X_b$ is infinite, we are done. Thus, we may assume that each $X_b$ is finite, and we can choose an infinite subset $Y'$ of $Y$ so that $X_b$ is distinct for each $b\in Y'$. We note that, by the stability of each extension $X_b$, for each $b,b'$ there is some element of $X_b$ which attacks some element of $X_{b'}$.

    We proceed by recursion to define a sequence of sets $S_i$. We begin with $S_0=D$. We will maintain the inductive hypothesis (\IH) that the set $C_i:=\{b\in Y' : S_i\subseteq X_b\}$ is infinite.

    \begin{lemma}\label{Make It Stable}
        For any $S_i$ and $z$ any argument there is some $x$ so that $x$ either attacks $z$ or $x=z$ so that letting $S_{i+1}=S_i\cup \{x\}$ preserves the \emph{\IH}.
    \end{lemma}
    \begin{proof}
        For each $b\in C_i$, by stability of $X_b$, either some $x\in X_b$ attacks $z$ or $z\in X_b$. Since only finitely many elements attack $z$, there is some $x$ which either is $z$ or attacks $z$ which is contained in infinitely many $X_b$ for $b\in C_i$. Then $S_{i+1}=S\cup \{x\}$ preserves  \IH.
    \end{proof}

    We proceed to define $S_i$ by iteratively applying Lemma \ref{Make It Stable}. We note that each $S_i$ is conflict free, since each $X_b$ is conflict free and there are always infinitely many $b$ so $S_i\subseteq X_b$. Thus $S:=\bigcup_i S_i$ is conflict-free. Further, for each $z\in A_\F$, we ensure at some stage that $S$ contains an element which either attacks $z$ or is $z$. Thus $S$ is stable. Finally, we note that $S$ cannot be finite. If $S$ were finite, then there would be some $i$ so that $S=S_i$, but then $S_i$ is a stable set contained in infinitely many different stable sets. But if $S$ is stable then no proper superset of it can be conflict-free, so this is impossible. Note that since each $S_i$ is contained in infinitely many $X_b$, which are each disjoint from $E$, it follows that $S\cap E=\emptyset$.
\end{proof}

\infstbLower*
\begin{proof} 
Let $(\H_n)_{n\in\omega}$ be the uniformly computable sequence of computably finitary AFs defined in the proof of Theorem \ref{infad lower} to prove that $\ex^\icf_\infad$ is $\Pi^0_2$-hard. Recall that each $\H_n$ is made of disjoint stars (AFs consisting of infinitely pairwise conflict-free arguments which are all attacked by another argument, called the center of the star). Furthermore, $\H_n$ contains exactly $|W_n|+1$ many stars. It is easy to see that $\H_n$ has only one stable extension, consisting of the set of centers of the stars; such a stable extension is infinite iff $W_n$ is infinite. Hence, we defined a reduction from $\Inf$ to $\ex^\icf_\infstb$, which proves that the latter set is $\Pi^0_2$-hard.
    
To see that $\cred^\icf_\infstb$ is also $\Pi^0_2$-hard, let $a$ be a center of a star in $\H_n$: we have that $a$ belongs to an infinite stable extension of $\H_n$ iff $n\in \Inf$.  
Similarly, if we choose an argument $b$ which is not the center of a star, then $b$ won't be in any stable   extension of $\H_n$. Thus, $b$ is skeptically accepted iff there is no infinite stable extension iff $n\in \Fin$, showing the $\Sigma^0_2$-hardness of $\skep^\icf_\infstb$.
\end{proof}

\uniinfstb*
\begin{proof}
It follows from the complexity of $\cred^\icf_\infstb$ and $\skep^\icf_\infstb$ (Theorem \ref{infstb lower}) that    $\uni^\icf_\infstb$ is $\Pi^0_3$, since a given computably finitary $\F$ has an infinite stable extension iff $\F\in \ex^\icf_\infstb$ and $(\forall x\in A_{\F}) (x\in \cred^\icf_\infstb(\F) \rightarrow x \in \skep^\icf_\infstb(\F))$.

    Fix $X$ a $\Pi^0_3$ set. Since $\Fin$ is $\Sigma^0_2$-complete, we can find a uniformly computable sequence of c.e.\ sets $(V^n_i)_{i,n\in\omega}$ so that $n\in X$ iff $V^n_i$ is finite for each $i$. For each $n\in \omega$, we construct $\G_n$ with $A_{\G_n}=\{a_i,b_i : i\in \omega\}$. We let $R_{\G_n}$ contain exactly the following attack relations:
    \begin{itemize}
    \item $a_i\att b_i$ and $b_i\att a_i$, for all $i\in\omega$;
    \item $a_i\att a_{i+j}$ and $a_i \att b_{i+j}$, if $j$ is a non-expansionary stage for $V^n_i$.
    \end{itemize}
    
    Note that $S=\{b_i : i\in \omega\}$ is a stable extension. 
    If $X$ were any other infinite stable extension, then $X$ contains an argument $a_i$. 
    But then $V^i_n$ must be infinite, since $X$ is infinite and $a_{i+j}$ or $b_{i+j}$ being in $X$ implies that $j$ is an expansionary stage for $V^i_n$. Thus $\G_n\notin\uni^\icf_\infstb$ implies $n\notin X$. On the other hand, if $n\notin X$, then there is some $V^i_n$ which is infinite, so there are infinitely many $j$ so that $a_i\not\att b_j$. Note that $\{a_i\}\cup \{b_j : a_i\not\att b_j\}$ is a second stable extension, so $\G_n\notin \uni^\icf_\infstb$. 
    Thus $\G_n\in \uni^\icf_\infstb$ iff $n\in X$, giving a reduction of $X$ to $\uni^\icf_\infstb$.
\end{proof}

\end{document}